\documentclass[10pt,journal,compsoc]{IEEEtran}
\IEEEoverridecommandlockouts
\usepackage[tight,footnotesize]{subfigure}
\usepackage{comment}
\usepackage{cite}
\usepackage{amsmath,amssymb,amsfonts}
\usepackage{algorithmic}
\usepackage{graphicx}
\usepackage{textcomp}
\usepackage{xcolor}
\usepackage{epsfig}
\usepackage{multirow}
\usepackage{graphicx}
\usepackage[tight,footnotesize]{subfigure}
\usepackage{diagbox}
\usepackage{tabularx}
\usepackage{color}
\usepackage{amsmath}
\usepackage[ruled,vlined,linesnumbered]{algorithm2e}
\usepackage{color}
\usepackage{amsthm}
\def\BibTeX{{\rm B\kern-.05em{\sc i\kern-.025em b}\kern-.08em
    T\kern-.1667em\lower.7ex\hbox{E}\kern-.125emX}}
    
\newtheorem{definition}{Definition}
\newtheorem{theorem}{Theorem}
\usepackage{multirow}

\begin{document}

\title{
  A Lightweight and Accurate Spatial-Temporal Transformer for Traffic Forecasting
}

\author{Guanyao Li, Shuhan Zhong, S.-H.~Gary Chan%~\IEEEmembership{Senior~Member,~IEEE,}
  \\ Ruiyuan Li, Chih-Chieh Hung, Wen-Chih Peng~%\IEEEmembership{Member,~IEEE}% <-this % stops a space
\IEEEcompsocitemizethanks{\IEEEcompsocthanksitem Guanyao Li, Shuhan Zhong, and S.-H.~Gary Chan are with the Department
of Computer Science and Engineering, The Hong Kong University of Science and Technology. \\ E-mail: \{gliaw, szhongaj, lxiangab, gchan\}@cse.ust.hk
\IEEEcompsocthanksitem Ruiyuan Li is with College of Computer Science, Chongqing University. \\ E-mail: liruiyuan@cqu.edu.cn
\IEEEcompsocthanksitem Chih-Chieh Hung is with Department of Computer Science and Engineering, National Chung Hsing University. \\ E-mail: smalloshin@email.nchu.edu.tw
\IEEEcompsocthanksitem Wen-Chih Peng is with Department of Computer Science, National Yang Ming Chiao Tung University. \\ E-mail: wcpeng@g2.nctu.edu.tw}
%\thanks{Manuscript received April 19, 2005; revised August 26, 2015.}
}

\newcommand{\FRAMEWORK}{ST-TIS}
\newcommand{\FRAMEWORKLONG}{\textbf{S}patial-\textbf{T}emporal \textbf{T}ransformer with \textbf{i}nformation fusion and region \textbf{s}ampling for traffic forecasting}
\newcommand*{\modelname}{\FRAMEWORK}

\IEEEtitleabstractindextext{
\begin{abstract}
We study the forecasting problem for traffic with dynamic, possibly periodical, and joint spatial-temporal dependency between regions.  Given the aggregated inflow and outflow traffic
  of regions in a city from time slots $0$ to $t - 1$, we predict the traffic at time $t$ at any region.  Prior arts in the area often consider the spatial and temporal dependencies in a decoupled manner, or
are rather computationally intensive in training  with a large number of hyper-parameters to tune.

We propose \modelname{}, a novel, lightweight and accurate \FRAMEWORKLONG{}.  
\modelname{} extends the canonical Transformer with  information fusion and region sampling. The information fusion module captures the
%dynamic, joint and possibly periodical
complex spatial-temporal dependency between regions.
The region sampling module is to improve the efficiency and prediction accuracy, cutting
the computation complexity for dependency learning from $O(n^2)$ to $O(n\sqrt{n})$, where $n$ is the number of regions.
With far fewer parameters than state-of-the-art models, 
\modelname{}'s offline training is significantly faster in terms of tuning and computation~(with a reduction of up to $90\%$ on training time and network parameters).  Notwithstanding such training efficiency, %\modelname{} achieves much higher accuracy in online prediction.
extensive experiments show that \modelname{} is substantially more accurate in online prediction 
than state-of-the-art approaches (with an average improvement of $9.5\%$ on RMSE, and $12.4\%$ on MAPE compared to STDN and DSAN). %We also provide case studies to demonstrate the interpretability and effectiveness of \modelname{} to capture the region correlation, temporal dependency, and irregularity effect. 
\end{abstract}

}

\maketitle

\section{Introduction}
\label{introduction}

Traffic forecasting is to predict the
inflow (i.e., the number of arriving objects per unit time) and outflow (i.e., the number of departing objects per unit time) 
of any region in a city at the next time slot.
The objects can be people, vehicles, goods/items, etc. 
Traffic forecasting has important applications in
%resource allocation for
%to meet demand in a timely manner, with
%significant social and commercial applications
transportation, %online-to-offline commerce or delivery, smart
retails,
public safety, city planning, etc~\cite{zheng2014urban,jiang2021graph}.
For example, with traffic forecasting, a taxi company may dispatch taxis in a timely manner to meet the supply and demand in different regions of a city.
Yet another example is bike sharing, where the company may want to balance bike supply and demand at dock stations (regions) based on such forecasting.
%An online delivery platform can pre-allocate couriers based on the forecasting result of food delivery requests. 
%A company can allocate sales promoters to different venues based on the forecasting result of crowd to maximize the influence. 
%Furthermore, it also has crucial applications in public safety~\cite{zheng2014urban}, city planning~\cite{kriegel2010route}, and traffic management~\cite{dunkel2011event}.

Although there has been much effort on deep learning to improve the prediction accuracy of the state-of-the-art forecasting models, progressive improvements on benchmarks have been correlated with an increase in the number of parameters and the amount of training resources required to train the model, making it costly to train and deploy large deep learning models~\cite{menghani2021efficient}. Therefore, a lightweight and training-efficient model is essential for fast delivery and deployment.

%Given 
%%the historical (aggregated) inflow and outflow data of different regions from time slots $0$ to $t-1$ (with slot size of, say, 30 minutes),  we seek a ``small'' training model with substantially fewer parameters to predict the inflow and outflow at any region at $t$,
%which naturally leads to
%efficiency in tuning, memory, and computation time. Despite its %lightweight, 
%it should also achieve higher accuracy than the state-of-the-art approaches in its online predictions. (Note that even though we consider predicting for the next time slot, our work can be straightforwardly extended to any future time slot by successive application of the algorithm.)
%For example, the cost of trying combinations of different hyper-parameters is computationally expensive and it highly relies on training resources. 
%If one is training a large model from scratch on either with limited or costly training resources, developing models that are designed for training efficiency would help.

In this work, we study the following spatial-temporal traffic forecasting problem:
Given 
the historical (aggregated) inflow and outflow data of different regions from time slots $0$ to $t-1$ (with slot size of, say, 30 minutes),  
what is the design of a  training-efficient model to accurately
predict the inflow and outflow of any region at time $t$?
(Note that even though we consider predicting for the next time slot, our work can be straightforwardly extended to any future time slot by successive application of the algorithm.)
We seek a ``small'' training model with substantially fewer parameters,
which naturally leads to
efficiency in tuning, memory, and computation time. Despite its
training efficiency, our lightweight model
lightweight, 
it should also achieve higher accuracy than the state-of-the-art approaches in its online predictions.

Intuitively, region traffic is spatially and temporally correlated. As an example, the traffic of a region could be correlated with that of another with some temporal lag due to the travel time between them. Moreover, such dependency may be dynamic over different time slots, and it may have temporal periodic patterns.
This is the case for the traffic of office regions, which exhibit high correlation with the residence regions in workday morning but much less than at night or weekend.
To accurately predict the region traffic, it is hence significantly crucial to account for the dynamic, possibly periodical, and joint spatial-temporal~(ST) dependency between regions, no matter how far the regions is apart.
%These realistic traffic characteristics indicate that it is essential to capture the spatial-temporal dependency between regions in a joint manner and consider the dynamic and periodical characteristics for accurate traffic prediction.

Much effort has been devoted to capturing the dependency between regions for traffic forecasting. While commendable, most  works consider spatial and temporal dependency separately with independent processing modules~\cite{zhang2016dnn,zhang2017deep,yao2018deep,yao2019revisiting,yu2018spatio,geng2019spatiotemporal}, which can hardly capture their joint nature in our current setting. 
%Among these works, most
Some recent works apply canonical Transformer~\cite{vaswani2017attention} to capture region dependency~\cite{zhou2020spatiotemporal,lin2020interpretable,lin2020preserving,xu2020spatial}.
%replacing convolution neural networks, graph convolution neural networks, and recurrent neural networks, etc.  
While impressive,  canonical Transformer limits the training efficiency %and prediction accuracy. This is
because it learns a region's embedding as the weighted aggregation of all the other regions based on their computed attention scores. This results in $O(n^2)$ computation complexity per layer, where $n$ is the number of regions. Moreover, it has been observed that the attention scores from the canonical Transformer have a long tail distribution~\cite{zhou2021informer}.
We illustrate this in Figure~\ref{fig:long-tail phenomenon} using a taxi dataset collected in New York City. We split New York City into $10 \times 20$ regions and visualize the attention scores~(after a SoftMax operation) between a target region~(i.e., the region $(6,4)$) and other regions.
%in Figure \ref{fig:long-tail phenomenon}.
Clearly, most regions have very small attention scores~(i.e., the long-tail phenomenon), with the attention scores of more than $60\%$ of regions being less than 0.004.
%, but these irrelevant regions contribute more than $40\%$ of embedding for the target region.
Such a long-tail effect may introduce noise for the region embedding learning and  degrade the prediction performance.

\begin{figure}%[!t]
    \centering
    \includegraphics[width=.75\linewidth]{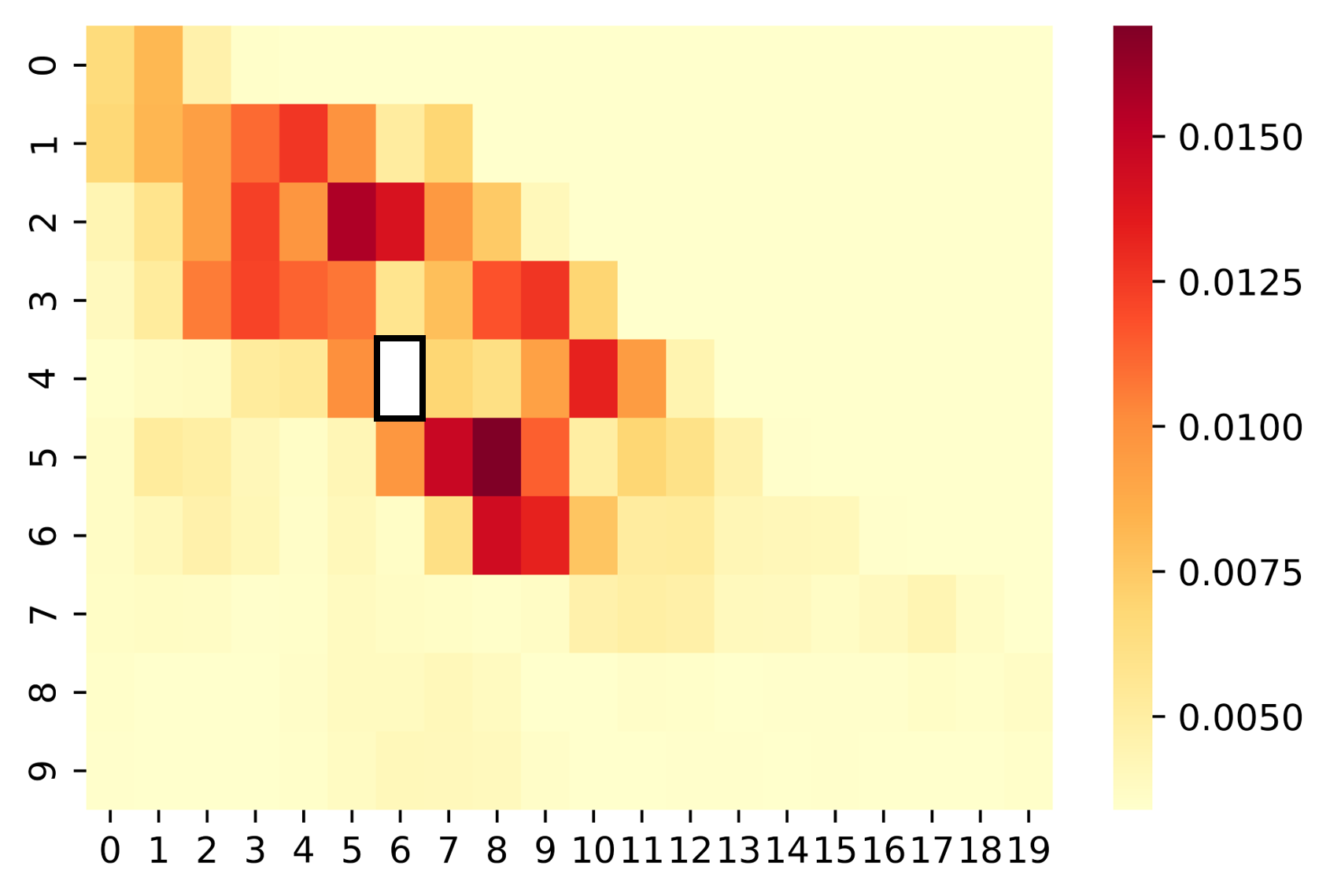}
    \caption{Visualization of attention scores between a target region~(6,4) and other regions. The color of a cell $(x_i,y_i)$ indicates the dependency of $(6,4)$ on $(x_i,y_i)$, where a darker color indicates stronger dependency.}
    \label{fig:long-tail phenomenon}
\end{figure}

We propose \modelname{}, 
a novel, small, efficient and accurate \FRAMEWORKLONG{}. Given the historical~(aggregated) inflow and outflow of regions from $0$ to $t-1$, \modelname{} predicts the inflow and outflow of any region at $t$, without relying on the transistion data between regions. With a small set of parameters,
\modelname{} is efficient to train~(offline phase).  It extends the canonical Transformer with  novel information fusion and region sampling strategies to learn the dynamic and possibly periodical spatial-temporal dependency between regions in a joint manner, hence achieving high prediction accuracy~(online phase).
%, i.e., inflow and outflow, at any regions for the next time slot.
%In contrast to many other training models based on deep learning,
%we seek a ``small'' model with substantially fewer parameters,
%which naturally leads to efficiency in tuning, memory and training time. Moreover, we propose learning the spatial and temporal dependency between regions in a joint manner, which considers the real traffic characteristic and is essential for accurate traffic forecasting. 
%To further boost its practicality and ease of deployment,
%we use only historical inflow and outflow data of regions, without the need for transition or flow data between regions, nor any additional external data such as POI data, weather data, irregularity information, etc.
%Despite its training efficiency, our lightweight model should achieve higher accuracy in its online predictions.

%Specifically, dependency learning consists of two phases. 
%First, considering that spatial-temporal dependency is dynamic over time, \modelname{} leverages an efficient Transformer encoder with information fusion and region sampling to jointly capture the spatial-temporal dependency at any individual time slot. After that, to capture the periodical characteristic of spatial-temporal dependency, \modelname{} employs an attention mechanism to automatically learn the periodical pattern of the spatial-temporal dependency from multiple historical time slots.
\modelname{} makes the following contributions: 

\begin{itemize}
    \item {\em A data-driven Transformer scheme for dynamic, possibly periodical, and joint spatial-temporal dependency learning.} \modelname{} jointly considers the spatial-temporal dependencies between regions, rather than considering the
      two dependencies sequentially in a decoupled manner. In particular, \modelname{} considers the dynamic spatial-temporal dependency for any individual time slot with an information fusion module, and also the possibly periodical characteristic of spatial-temporal dependency from multiple time slots using an attention mechanism. Moreover, the dependency learning is data-driven, without any assumption on spatial locality. Due to its design, \modelname{} is small (in parameter footprint), fast (in training  time), and accurate (in prediction).       

    \item {\em A novel region sampling strategy for computationally efficient dependency learning.} \modelname{} leverages the Transformer framework~\cite{vaswani2017attention} to learn region dependency. To address the quadratic computation issue and mitigate the long-tail effect, it employs a novel region sampling strategy to generate a connected region graph and learns the dependency based on the graph. The dependencies between any pair of regions~(both close and distant dependencies) are guaranteed to be considered in \modelname{} via information propagation, and the computational complexity is reduced from $O(n^2)$ to $O(n\sqrt{n})$, where $n$ is the number of regions.

    \item {\em Extensive experimental validation:} We evaluate \modelname{} on two large-scale real datasets of taxi and bike sharing. Our results shows that \modelname{} is substantially more accurate than the state-of-the-art approaches, with a significant improvement in RMSE and MAPE~(an average improvement of $9.5\%$ on RMSE, and $12.4\%$ on MAPE compared to STDN and DSAN). Furthermore, it is much more lightweight than most state-of-the-art models, and is ultra fast for training~(with a reduction of $46\% \sim 95\%$ on training time and $23\% \sim 98\%$ on network parameters). %Through case studies, we demonstrate that \modelname{} is interpretable, and effectively considers region correlations, temporal dependencies, and irregularities.

\end{itemize}

The remainder of this paper is organized as follows.
We first discuss related works in 
Section~\ref{related_work}.  After preliminaries in Section~\ref{sec:preliminary}, we detail \modelname{} in Section~\ref{sec:method}. We present the experimental settings and results in Section~\ref{sec:evaluation}, and conclude in Section~\ref{sec:conclusion}.

\section{Related Works}
\label{related_work}

%In this section, we present some representative related works on traffic forecasting.
%and compare them with \modelname{}. %We categorize them into two groups: conventional approaches and deep learning approaches.

Traffic forecasting has raised much attention in both
academia and industry due to its social and commercial
values. Some early traffic forecasting works propose using regression models, such as auto-regressive integrated moving average~(ARIMA) models~\cite{shekhar2007adaptive,pan2012utilizing,moreira2013predicting,abadi2014traffic} and non-parametric region models~\cite{smith2002comparison,silva2015predicting}. All of these works consider temporal dependency, but they have not considered the spatial dependency between regions. Some other works extract features from heterogeneous data sources~(e.g., POI, weather, etc.), and use machine learning models such as
Support Vector Machine~\cite{zhang2007forecasting},
%Gaussian Processes~\cite{sun2010variational}, 
%Bayesian model~\cite{xu2014accurate}, 
Gradient Boosting Regression Tree~\cite{li2015traffic}, 
and linear regression model~\cite{tong2017simpler}. Despite of the encouraging results, they rely on manually defined features and have not considered the joint spatial-temporal dependency.

In recent years, deep learning techniques have been employed to study spatial and temporal correlations for traffic forecasting. Most existing works consider the spatial and temporal dependency in a decouple manner~\cite{zhang2016dnn,zhang2017deep,yao2018deep,yao2019revisiting,yu2018spatio,geng2019spatiotemporal}, which can hardly capture their joint effect. For spatial dependency, Convolution Neural Network~(CNN), Graph Neural Network~(GNN), and Transformer have been widely applied. Regarding temporal dependency, Recurrent Neural Network~(RNN) and its variants such as Long Short Term Memory~(LSTM) and Gated Recurrent Unit~(GRU) have been extensively studied. Compared with these works, \modelname{} considers the spatial and temporal dependency in a joint manner.
%Fully connected networks and residual structures are used in the work~\cite{wang2017deepsd} to fuse multiple data sources for supply-demand prediction. However, the spatial correlation is not considered in the work.

CNN has been applied in many works to capture dependencies between close regions~\cite{zhang2016dnn,zhang2017deep,yao2018deep,yao2019revisiting,li2020autost}. 
In these works, a city is divided into some connected but non-overlapping grids, and the traffic in each grid is then predicted.
%and traffic data is embedded into the grids like pixels of an image, so as to facilitate the convolution operation. 
However, these works cannot be used for fine-grained flow forecasting at an individual location~\cite{he2020towards}, such as predicting flow for a docked bike-sharing station or a subway station. Moreover, CNN can hardly capture distant traffic dependency due to its relatively small receptive field~\cite{yao2019graph,li2019deepgcns}.
%Moreover, as CNN focuses on capturing local correlations, these works
%cannot capture distant spatial correlations.
%cannot be extended to consider the correlation with distant regions, which is crucial in some scenarios. 

Some other works use GCN to capture spatial dependency~\cite{yu2018spatio,geng2019spatiotemporal,pan2019urban,wang2020traffic,fang2021spatial,li2021spatial,song2020spatial,shi2020predicting,yuan2021effective}. In these works, a city is represented as a graph structure, and convolution operations are applied to aggregate spatially distributed information in the graph.  In each aggregation layer, a region would aggregate the embedding of its neighbouring regions in the graph. However, these works highly rely on the graph structure for dependency learning. Prior works usually construct graphs based on the distance between regions or road network, based on the locality assumption~(i.e., close regions have higher dependency). They have to stack more layers to learn dependency if the distance between two regions in the graph is long, and it ends up with an inefficient and over-smoothing model~\cite{fang2021spatial}. In recent years, Transformer~\cite{vaswani2017attention} has been applied for traffic forecasting~\cite{zhou2020spatiotemporal,lin2020interpretable,lin2020preserving,xu2020spatial}. The canonical Transformer can be seen as a special graph neural network with a complete graph, in which any pair of regions are connected. Consequently, the dependencies between both close and distant regions could be considered. Moreover, the self-attention mechanism and the network structure of Transformer have been demonstrated to be powerful in many prior works. However, the computational complexity of each aggregation round is $O(n^2)$ for Transformer where $n$ is the number of regions, while that for GNN is $O(E)$ where $E$ is the number of edges in the graph~($E \le n^2$). Furthermore, as a region may only have a strong dependency on a small portion of regions, aggregating the embedding of all regions would introduce noise and degrade its performance. In \modelname{}, we propose a region connected graph~(i.e., there exists a path between any pair of regions in the graph), in which the degree of any node~(i.e., region) is $O(\sqrt{n})$ and the distance between any two regions in the graph is no more than $2$. We extend the canonical Transformer with the proposed region connected graph, so that it can inherit the advantage of efficiency and effectiveness from both GNN and Transformer.

RNN and its variants such as LSTM and GRU~\cite{hochreiter1997long,chung2014empirical} have been used to capture temporal dependency~\cite{li2017diffusion,yao2018deep,shi2020predicting,yuan2021effective}. 
However, the performance of RNN-based models deteriorates rapidly as the length of the input sequence increases~\cite{cho2014properties}.
Some works incorporate the attention mechanism~\cite{bahdanau2014neural} to improve their capability of modeling long-term temporal correlations~\cite{pan2019urban,geng2019spatiotemporal,yao2019revisiting,liang2018geoman}. 
Nevertheless, RNN-based networks are widely known to be difficult to train and are computationally intensive~\cite{yu2018spatio}. As recurrent networks generate the current hidden states as a function of the previous hidden state and the input for the position, they are in general more difficult to be trained in parallel. 
To address the issue, the self-attention mechanism is proposed as the replacement of RNN to model sequential data~\cite{vaswani2017attention}. It has enjoyed success in capturing temporal correlations for traffic forecasting~\cite{li2019forecaster,zhou2020spatiotemporal,lin2020interpretable,lin2020preserving}. Compared with RNN-based models, self-attention models can directly model long-term temporal interactions,
but the computational complexity of using self-attention for temporal dependency learning in existing works is $O(q^2)$, where $q$ is the number of historical time slots. Compared with them, \FRAMEWORK{} is conditional on the spatial-temporal dependency at any individual slot to generate weights for different time slots, so its computational complexity is $O(q)$ in our work. In addition, the temporal dependency is jointly considered with the spatial dependency in \modelname{}, instead of in a decouple manner.

Some variants of Transformer have been proposed to address the efficiency issues of the canonical Transformer, such as LogSparse~\cite{li2019enhancing}, Reformer~\cite{kitaev2019reformer}, Informer~\cite{zhou2021informer}, etc. While impressive, these approaches cannot be used in the scenario of capturing spatial-temporal dependency between regions for traffic forecasting we are considering in this work.

\section{Preliminaries}
\label{sec:preliminary}
%In this section, we introduce some terms that will be used in the following discussion, and define the problem.

%Following the previous works~\cite{zhang2017deep,yao2019revisiting}, we partition a city into $n$ non-overlapping grids~(a.k.a. regions) with equal size $\{\ell_1, \ell_2, ..., \ell_n\}$.

%We partition a city into a set of regions $\{\ell_1, \ell_2, ..., \ell_n\}$. We define the length of a time slot as $30$ minutes, and partition the time period to some continuous time slots. 

\subsection{Problem formulation}
\begin{definition}
(Region) The area~(e.g., a city or subway route) is partitioned into $n$ non-overlapping regions. We use $R = \{r_1, r_2, \dots, r_n\}$ to denote the partitioned regions, in which $r_i$ denotes the $i$-th region. 
\end{definition}

The way to partition an area is {\em flexible} for \modelname{}, e.g., grid map, road network, clustering, or train/bus/bike stations, etc. 

\begin{definition}
(Traffic data) We use $\mathcal{I}$ and $\mathcal{O}$ to denote the inflow and outflow data of all regions over time, respectively. Specifically, $\mathcal{I}^t \in \mathbb{R}^{1 \times n }$ and $\mathcal{O}^t \in \mathbb{R}^{1 \times n }$ is the inflow and outflow of the $n$ regions at time $t$. Moreover, $\mathcal{I}_i^{t}$ and $\mathcal{O}_i^{t}$ is the inflow and outflow of region $r_i$ at time $t$~(i.e., the number of objects arriving at/departing from the region $r_i$ at time slot $t$). Furthermore, we use $\mathcal{I}_i^{t':t} = [\mathcal{I}_i^{t'}, \mathcal{I}_i^{t'+1}, \dots, \mathcal{I}_i^{t}]$  to denote the inflow  of $r_i$ from $t'$ to $t$. Similarly, $\mathcal{O}_i^{t':t} = [\mathcal{O}_i^{t'}, \mathcal{O}_i^{t'+1}, \dots, \mathcal{O}_i^{t}]$ indicates the outflow of $r_i$ from $t'$ to $t$.
\end{definition}

The formal formulation of the traffic forecasting problem is as follows:

\begin{definition}
(Traffic Forecasting) Given the traffic data of regions from $0$ to $t-1$, namely $\mathcal{I}^{0:t-1}$ and $\mathcal{O}^{0:t-1}$, the traffic forecasting problem is to predict the inflow and outflow of any region at $t$, namely $\mathcal{I}^{t}$ and $\mathcal{O}^{t}$.
\end{definition}

\subsection{\modelname{} Overview}
\label{sec:overview}

\begin{figure}%[!t]
  \centering
  \includegraphics[width=0.75\linewidth]{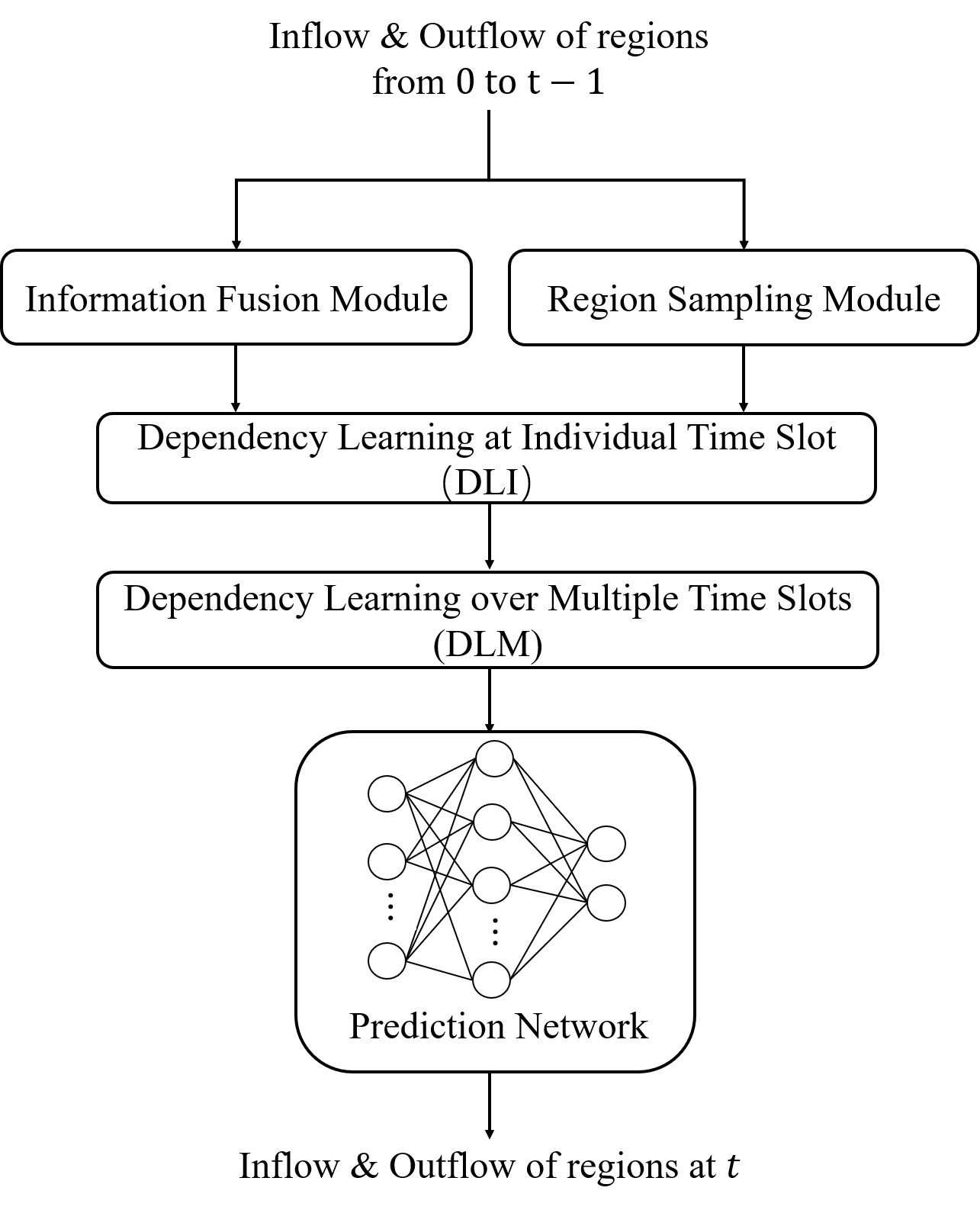}
  \caption{\modelname{} overview. }
  \label{fig:overview_of_STREST}
\end{figure}

We overview the proposed \modelname{} in Figure \ref{fig:overview_of_STREST}. Given the traffic data of all regions from time slots $0$ to $t-1$, \modelname{} is an end-to-end model to capture the spatial-temporal dependency and predict the inflow and outflow for any region at $t$.
%\modelname{} first learns the correlation between any pair of regions at any time slot, and it employs a temporal attention module to efficiently capture the temporal dependency and anomaly effect, followed by a prediction model to predict the inflow and outflow for any region. As shown in Figure \ref{fig:overview_of_ST_SAM}, \modelname{} is an end-to-end model for traffic forecasting. 
There are five modules in \modelname{}. We explain the design goals and the relationship among modules as follows:
\begin{itemize}
  \item {\em Information Fusion Module:} A key to accurately predicting the traffic for regions is capturing the dynamic spatial-temporal~(ST) dependency between regions in a joint manner. To this end, \modelname{} employs a information fusion module to learn the spatial-temporal-flow~(STF) embedding by encoding its spatial, temporal, and flow information for any individual region at a time slot. 
  \item {\em Region Sampling Module:} To address the issues of quadratic computational complexity and long tail distribution of attention scores for the canonical Transformer, \modelname{} uses a novel region sampling strategy to generate a region connected graph.
  %, in which the degree of each node~(i.e., region) is $O(\sqrt{n})$ and the distance between any two nodes in the graph is less than $2$. 
  The dependency learning would be based on the generated graph. 
  \item {\em Dependency Learning for Individual Time Slot~(DLI):} Given the STF embedding at a time slot and the region connected graph, \modelname{} extends the canonical Transformer to jointly capture the dynamic spatial-temporal dependency between regions at the time slot. As the spatial and temporal information has been both encoded in the STF embedding, the joint effect of spatial-temporal dependencies between regions could be captured. Moreover, with the region connected graph, only the attention scores between neighbouring nodes~(i.e., regions) are computed, and only the embedding of one's neighbouring regions are aggregated. Dependency between non-neighbouring nodes is considered via information propagation between multiple layers in the network. Consequently, it cuts the computational complexity of a layer from $O(n^2)$ to $O(n \times \sqrt{n})$ where $n$ is the number of regions, and addresses the issue of the long-tail effect for dependency learning.
  \item {\em Dependency Learning over Multiple Time Slots~(DLM):} Given the region embedding from DLI at multiple time slots, DLM captures the periodic patterns of spatial-temporal dependency using the attention mechanism. The influence of spatial-temporal dependency at historical time slots on the predicted time slot is hence considered in \modelname{}. The learning process is data-driven, without any prior assumption of traffic periods.
  \item {\em Prediction Network~(PN)}: Given the results from DLM, \modelname{} uses a fully connected neural network to predict the inflow and outflow of regions~($\mathcal{I}^t$ and $\mathcal{O}^t$) simultaneously.
\end{itemize}

\section{\FRAMEWORK{} Details}
\label{sec:method}

We present the details of \modelname{} in this section. We first elaborate the information fusion module in Section \ref{sec:information fusion module}, followed by the region sampling module in Section \ref{sec:region sampling module}. After that, we introduce the dependency learning for individual time slot~(DLI) in Section \ref{sec: dependency_learning_for_individual_time_slot}, and the dependency learning over multiple time slots~(DLM) in Section \ref{sec:dependency learning over multiple time slots}. Finally, we present the prediction network~(PN) in Section \ref{sec:forecasting_layer}.

\subsection{Information Fusion Module}
\label{sec:information fusion module}

To capture the dynamic joint spatial-temporal dependency for traffic forecasting, it is essential to fuse the spatial-temporal information for each region at any individual time slot. To this end, \modelname{} employs the information fusion module to learn one's spatial-temporal-flow~(STF) embedding by fusing its position, time slot, and flow information.

Given $n$ regions, we first use a one-hot vector  $\mathcal{S}_i \in \mathbb{R}^{1\times n}$ to represent a region $r_i$, in which only the $i$-th element in $\mathcal{S}_i$ is $1$ and otherwise $0$. After that, we encode the position information for a region $r_i$ as follows,
\begin{equation}
    \hat{\mathcal{S}}_i = \mathcal{S}_i \cdot W^{S} + b^{S}. 
\end{equation}
where $\hat{\mathcal{S}}_i \in \mathbb{R}^{1 \times d}$ is the spatial embedding of $r_i$, $W^{S} \in \mathbb{R}^{n \times d}$ and $b^{S} \in \mathbb{R}^{1 \times d}$ are learnable parameters, and $d$ is a hyperparameter.

In terms of temporal information, we first split a day into $o$ time slots and represent the $i$-th time slot using a one-hot vector $\mathcal{T}_i \in \mathbb{R}^{1 \times o}$. After that, we learn the temporal embedding by
\begin{equation}
    \hat{\mathcal{T}}_i = \mathcal{T}_i \cdot W^{T} + b^T,
\end{equation}
where $\hat{\mathcal{T}}_i \in \mathbb{R}^{1\times d}$ is the temporal embedding of the $i$-th time slot in a day, $W_T \in \mathbb{R}^{o \times d}$ and $b^T \in \mathbb{R}^{1 \times d}$ are learnable parameters.

Recall that the traffic of one region at $t_i$ may depend on that of another region at previous time slots due to the travel time between two regions. Thus, we use one's surrounding observations instead of solely using a snapshot to learn the dependency~\cite{li2019enhancing}. We define the surrounding observations of a region at a time slot $t_j$ as follows:
\begin{definition}
(Surrounding observations) For a region $r_i$ at a time slot $t_j$, its surrounding observations are defined as the flow in the $w$ previous time slots: $\{\mathcal{I}^{t_j-w}_i, \cdots, \mathcal{I}^{t_j-2}_i, \mathcal{I}^{t_j - 1}_i, \mathcal{O}^{t_j-w}_i, \cdots, \mathcal{O}^{t_j-2}_i,\mathcal{O}^{t_j-1}_i\}$.
\end{definition}

Note that by employing the surrounding observations, the temporal lag of the dependency between regions could be considered. Given the surrounding observations of $r_i$ at $t_j$, we first apply the 1-D convolution of kernel size $1 \times p$~($p < w$) with stride $1$ and $f$ output channels on its inflow and outflow surrounding observations to extract different patterns. The flow embedding $\mathcal{F}_i^{t_j} \in \mathbb{R}^{1\times d}$ of $r_i$ at $t_j$ is computed as 
\begin{equation}
    \mathcal{F}_i^{t_j} = (Conv^I(\mathcal{I}_i^{t_j-w:t_j-1}) || Conv^O(\mathcal{O}_i^{t_j-w:t_j-1})) \cdot W^F + b^F,
\end{equation}
where $||$ is the concatenation operation, $Conv^I(\cdot)$ and $Conv^O(\cdot)$ are the convolutional operation for inflow and outflow data respectively, $W^F \in \mathbb{R}^{ l \times d}$,  $b^F \in \mathbb{R}^{1 \times d}$ are learnable parameters, and $l = 2 \times f \times (w-k+1)$. 

Finally, we fuse the position, time slot and flow information to learn the STF embedding of a region $r_i$ at time $t_j$. We define that $t_j$ is the $M(t_j)$-th time slot in a day, where $M(\cdot)$ is an matching function. The fusion process is defined as 
\begin{equation}
\label{equ:STF_embedding}
    \mathcal{L}_i^{t_j} = (\hat{\mathcal{S}}_i + \hat{\mathcal{T}}_{M(t_j)} + \mathcal{F}_i^{t_j}) \cdot W^L + b^L_i,
\end{equation}
where $\mathcal{L}_i^j \in \mathbb{R}^{1\times d}$ is the STF embedding, $W_L \in \mathbb{R}^{d \times d}$ and $b^L_i \in \mathbb{R}^{1 \times d}$ are learnable parameters.

\subsection{Region Sampling}
\label{sec:region sampling module}

\begin{figure}%[!t]
    \centering
    \includegraphics[width=0.75\linewidth]{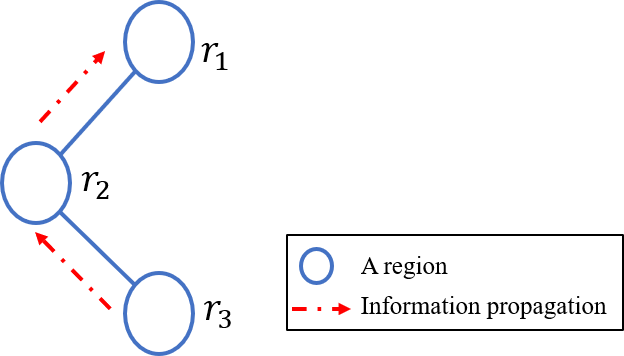}
    \caption{Information propagation in a graph.}
    \label{fig:information propagation}
\end{figure}

\begin{figure}%[!t]
    \centering
    \includegraphics[width=.75\linewidth]{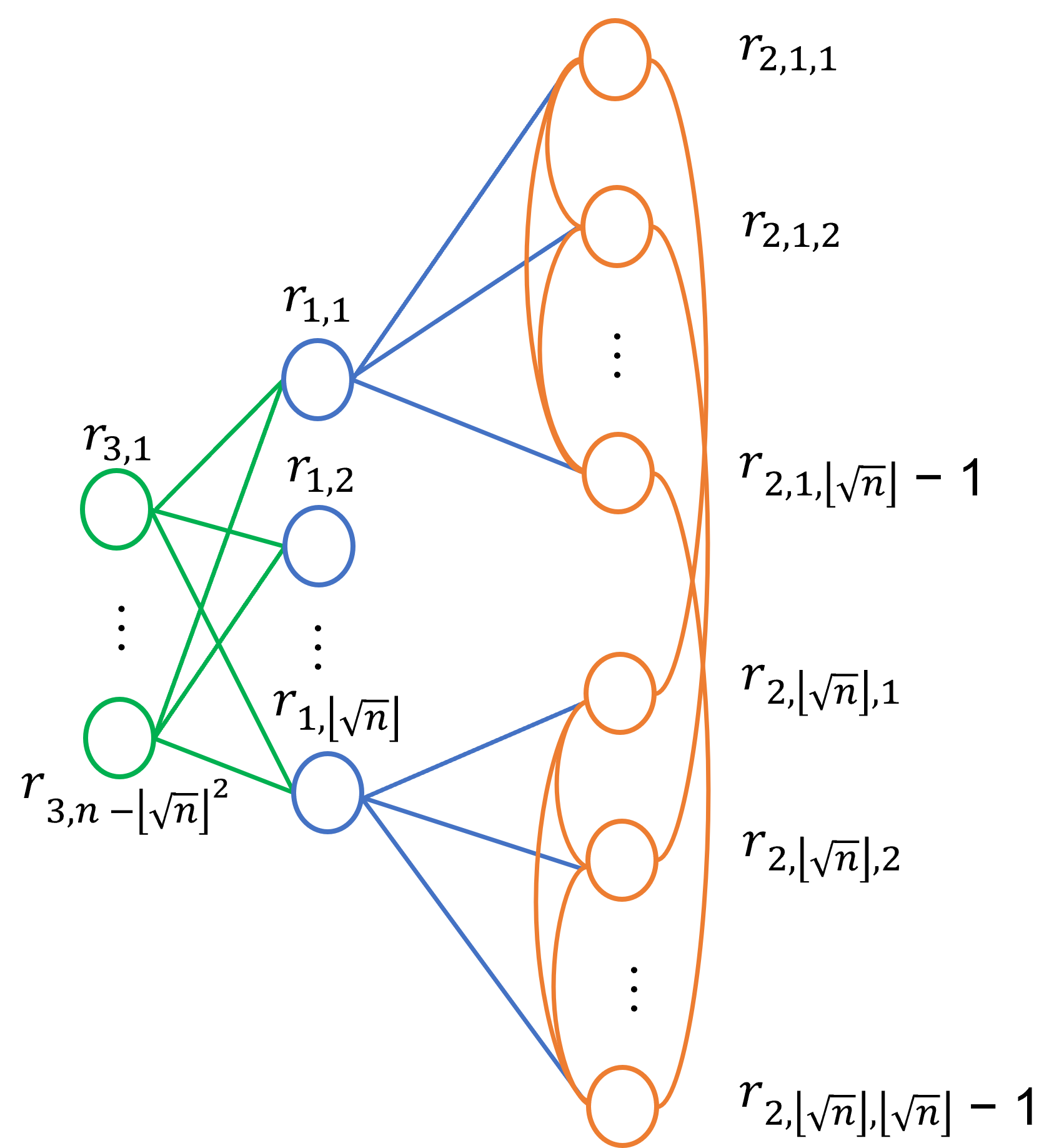}
    \caption{The process of connected graph generation.}
    \label{fig:graph generation}
\end{figure}

To capture a region's dependency on others~(both nearby and distant), a canonical Transformer computes the attention scores between the target region and all other regions, and aggregats the embedding of all other regions based on the computed attention scores. However, this results in the issues of quadratic computation and long-tail effect for dependency learning.

Fortunately, prior works have showed that information could be propogated between nodes in a graph via a multi-layer network structure~\cite{wu2020comprehensive}. Hence, with a proper graph structure and network structure, a region can aggregate the embedding of another even without directly evaluating their attention score.
We present a toy example in Figure \ref{fig:information propagation}, in which regions are represented as nodes in the graph and connected with egdes. In the first aggregation layer,  $r_1$ would capture the information from $r_2$, while $r_2$ would capture the information from $r_3$. Since the information of $r_3$ has been aggregated in $r_2$, $r_1$ could also capture it in the second aggregation layer without computing the attention score between $r_1$ and $r_3$. From this example, we conclude that a node's information can reach another with a $\beta$-layer aggregation operation if their distance is not more than $\beta$ in the graph. 

To address the limitations of the canonical Transformer, we propose to generate a connected graph~(i.e., there exists at least a path between any two nodes in the graph), in which the degree of any node (i.e., region) is not more than $c\times \sqrt{n}$ and the distance between any pair of nodes are not more than 2. In this way, for a target region, we only have to aggregate the embedding of $O(\sqrt{n})$ regions in an aggregation layer, and the influence from other regions can be captured in a two-layer aggregation process by information propagation. 
%In the neural network, each aggregation process could be computed with a network layer. Thus, the information propagation could be achieved with a multi-layer network structure.
With such design, we do not have to compute the attention scores between any pairs of regions and aggregate the embedding of all $n$ regions for a target region, so that the computation complexity is reduced to $O(n\sqrt{n})$ for each layer, and the long-tail issue is also addressed.

In this work, we present a heuristic approach for region connected graph generation. We first calculate the traffic similarity between any pair of regions. The calculation of the traffic similarity is flexible and it could be any similarity metric. As an example, we use DTW in this work to measure the similarity in terms of the average traffic over time slots of a day. We use $\mathcal{M} \in \mathbb{R}^{n \times n}$ to represent the similarity matrix, in which $\mathcal{M}_{i,j}$ is the similarity between $r_i$ and $r_j$.
Based on $\mathcal{M}$, the process of the connected graph generation is illustrated in Figure \ref{fig:graph generation}.

We first select top $\lfloor \sqrt{n} \rfloor$ regions without replacement, which have the largest sum of similarity with other regions, represented as $\{r_{1,1}$, $r_{1,2},\cdots,r_{1,\lfloor \sqrt{n} \rfloor}\}$. 
For any region $r_{1,i}$, we then select $\lfloor \sqrt{n} \rfloor - 1$ regions without replacement, which have the largest similarity between them and $r_{1,i}$, represented as $\{r_{2,i,1},r_{2,i,2},\cdots,r_{2,i,\lfloor \sqrt{n} \rfloor-1}\}$, and we connect $r_{1,i}$ to $r_{2,i,j}$~($j = 1,2,\cdots,\lfloor \sqrt{n} \rfloor-1$). 
After that, we connect a region $r_{2,i,j}$ to regions $r_{2,i,k}$ and $r_{2,u,j}$, where $k \in \{\{1,2,...,\lfloor \sqrt{n} \rfloor-1\} \backslash \{j\}\}$ and  $u \in \{\{1,2,...,\lfloor \sqrt{n} \rfloor\} \backslash \{i\}\}$. 
Finally, if $\sqrt{n} \notin \mathbb{Z}$, the remaining regions woule be connected to $r_{1,i}$ where $i \in \{1,2,\cdots,\lfloor \sqrt{n} \rfloor\}$, represented as $\{r_{3,1}, r_{3,2}, \cdots, r_{3,n - \lfloor \sqrt{n} \rfloor^2}\}$.
If $\sqrt{n} \in \mathbb{Z}$, we randomly select a region from $r_{1,i}$, and connect it to $r_{1,j}$, where $j \in \{1,2,\cdots,\lfloor \sqrt{n} \rfloor\} \backslash \{i\}$, represented as $r^{*}$.

\begin{theorem}
The degree of any node in the region connected graph is $O(\sqrt{n})$, and the distance between any two nodes in the graph is less than $2$.
\end{theorem}

\begin{proof}
%The distance between $r_{3,w}$~(or $r^*$) and $r_{1,i}$ and the distance between $r_{1,i}$ and $r_{2,i,j}$ are both $1$. Thus, the distance from  $r_{3,w}$ and $r_{1,i}$ to any other regions in the graph is less than $2$. Moreover, because the distance between $r_{2,i,j}$ and $r_{2,i,k}$ and that between $r_{2,i,j}$ and $r_{2,k,j}$ are $1$, $r_{2,i,j}$ can reach any other region within two hops.
The generation of the region connected graph ensures that a node is connected to at most $\max(2 \times \lfloor \sqrt{n} \rfloor - 2, n - \lfloor \sqrt{n} \rfloor^2 + \lfloor \sqrt{n} \rfloor - 1)$ other nodes, and hence the degree is $O(\sqrt{n})$.

The distance of $(r_{3,i},r_{1,j})$~(or $(r^*, r_{1,j})$) and $(r_{1,j},r_{2,j,k})$ is both $1$, where $i \in \{1,2,\cdots,n-\lfloor \sqrt{n} \rfloor^2\}$, $j \in \{1,2,\cdots,\lfloor \sqrt{n} \rfloor\}$, and $k \in \{1,2,...,\lfloor \sqrt{n} \rfloor - 1\}$. Thus, the distance between $r_{3,i}$~(or $r^*$) and any other region is no more than $2$ in the graph.

For region $r_{1,j}$, since the distance of $(r_{1,j}, r_{2,j,k})$ and $(r_{2,j,k}, r_{2,m,k})$ is both 1 where $k \in \{1,2,...,\lfloor \sqrt{n} \rfloor - 1\}$ and $m \in \{1,2,\cdots, \lfloor \sqrt{n} \rfloor\} \backslash \{j\}$, the distance of $(r_{1,j}, r_{2,m,k})$ is hence 2. Thus, the distance between $r_{1,j}$ and any other region is also no more than $2$.

In terms of $r_{2,j,k}$, as the distance of $(r_{2,j,k}, r_{2,m,k})$ and $(r_{2,j,k}, r_{2,j,v})$ is 1, where  $m \in \{1,2,\cdots, \lfloor \sqrt{n} \rfloor\} \backslash \{j\}$ and  $v \in \{1,2,\cdots, \lfloor \sqrt{n} \rfloor - 1\} \backslash \{k\}$, the distance between $r_{2,j,k}$ and any other regions is hence no more than $2$. Therefore, the distance between any two regions in the graph is less than $2$.
\end{proof}

Because the distance between any two regions in the proposed graph is less than $2$, the dependencies between any two regions could be considered if the layer number of the aggregation network is larger than $2$.

\subsection{Dependency Learning for Individual Time Slot~(DLI)}
\label{sec: dependency_learning_for_individual_time_slot}

Given the STF embedding of regions for time $t_j$ and the generated region connected graph, \modelname{} captures the spatial-temporal dependencies between regions based on an extended Transformer encoder. %Unlike the canonical Transformer, we aggregate the information os a region's neighbouring regions in the region graph, while the information of non-neighbouring regions are considered by the information propagation between layers.
Following the canonical Transformer, \modelname{} employs an multi-head attention mechanism, so that it could account for different dependencies between regions. For the $m$-th head, the attention score between $r_i$ and $r_v$ at $t_j$ is defined as 
\begin{equation}
\label{equ: attention equ}
    \mathcal{A}_m(r_i,r_v,t_j) = \frac{(\mathcal{L}_i^{t_j} \cdot W_{Q_m}) \cdot (\mathcal{L}_v^{t_j} \cdot W_{K_m})^T}{\sqrt{d}},
\end{equation}
where $\mathcal{L}_i^{t_j} \in \mathbb{R}^{1\times d}$ and $\mathcal{L}_v^{t_j} \in \mathbb{R}^{1 \times d}$ are the STF embedding of regions $r_i$ and $r_v$ at $t_j$~(Equation \ref{equ:STF_embedding}), $W_{Q_m} \in \mathbb{R}^{d \times d}$ and $W_{K_m} \in \mathbb{R}^{d \times d}$ are learnable parameters, and $d$ is a hyperparameter for the embedding size.

Unlike the canonical Transformer, we do not evaluate the attention scores between one region and all other regions. Instead, we only compute one's attention scores with its neighbouring regions in the region connected graph, and aggregate their embedding in terms of their attention score to update the region's embedding.
For the $m$-th head, the embedding of a region $r_i$ at time $t_j$ is then updated as
\begin{equation}
\begin{split}
    \hat{\mathcal{L}}^{t_j}_{i,m} &= \sum_{r_v \in Neigh(r_i)} softmax(\mathcal{A}_m(r_i,r_v,t_j)) \cdot \mathcal{L}^{t_j}_v \\ &= \sum_{r_v \in Neigh(r_i)} \frac{\exp(\mathcal{A}_m(r_i,r_v,t_j))}{\sum_{r_u \in Neigh(r_i)} \exp(\mathcal{A}_m(r_i,r_u,t_j))} \cdot \mathcal{L}^{t_j}_v,
\end{split}
\end{equation}
where $\hat{\mathcal{L}}^{t_j}_{i,m} \in \mathbb{R}^{1 \times d}$ is the output of the $m$-th head, $Neigh(r_i)$ is the neighbouring regions of $r_i$ in the graph, $\mathcal{A}_m(r_i,r_v,t_j)$ is the attention score defined in Equation \ref{equ: attention equ}. As the degree of any node is $O(\sqrt{n})$, the computation complexity of attention score evaluation and embedding aggregation is hence $O(n\sqrt{n})$ for all regions in a layer.

Finally, we concatenate the results of multi-heads and the embedding of $r_i$ is computed as 
\begin{equation}
    \hat{\mathcal{L}}^{t_j}_i = Concat(\hat{\mathcal{L}}^{t_j}_{i,1}, \cdots, \hat{\mathcal{L}}^{t_j}_{i,M}) \cdot W^O,
\end{equation}
where $Concat(\cdot)$ is the concatenation operation, $\hat{\mathcal{L}}^t_i \in \mathbb{R}^{1 \times d}$ is the embedding of $r_i$, $W_O \in \mathbb{R}^{(d \times M) \times 1}$ are learnable parameters and $M$ is the number of heads.

\begin{figure}%[!t]
    \centering
    \includegraphics[width=0.75\linewidth]{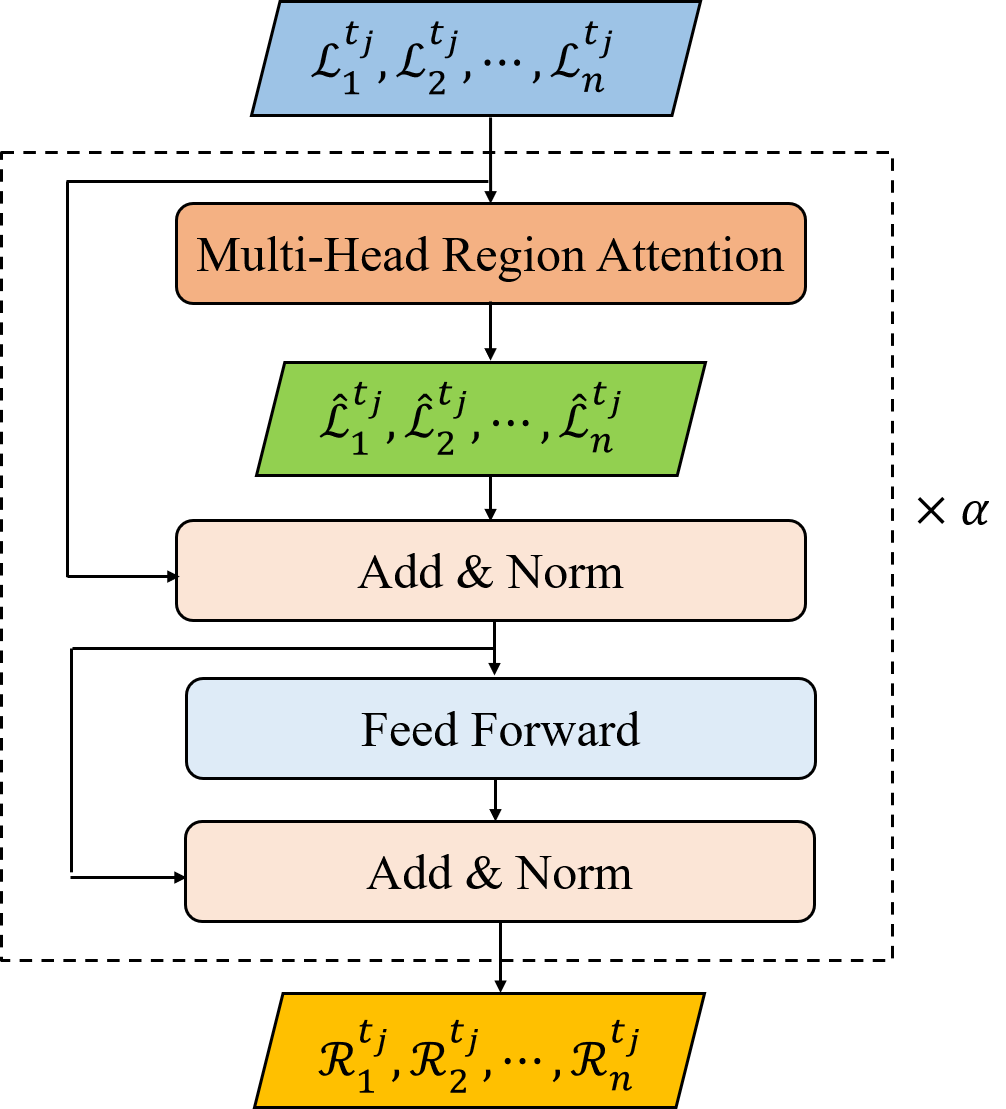}
    \caption{Processing of dependency learning for individual time slot.}
    \label{fig:DLI}
\end{figure}

Following the structure of Transformer\cite{vaswani2017attention}, the output of the multi-head region attention layers $\hat{\mathcal{L}}_i^{t_j}$ is then passed to a fully connected neural network~(Figure \ref{fig:DLI}). We also employ a residual connection between each of the two layers. 
As we discussed in Section \ref{sec:region sampling module}, the information propagation is achieved with a multi-layer network structure. Thus, we stack the layers $\alpha$ times~(the effect of $\alpha$ will be discussed in Section \ref{sec:layer_num}). We denote the final output of the DLI as $\mathcal{R}^{t_j} = \{\mathcal{R}_1^{t_j}, \mathcal{R}_2^{t_j}, \cdots, \mathcal{R}_n^{t_j}\}$, where $\mathcal{R}_i^{t_j}$ is the embedding of region $r_i$ at $t_j$.

Compared with the canonical Transformer, we only need to compute the attention scores and aggregate the embedding between adjacent nodes in the region connected graph. The computation complexity is hence reduced from $O(n^2)$ to $O(n\sqrt{n})$ for each layer.
%\textcolor{red}{Explain why the RAM is efficient.}

%RAM is significantly efficient and effective to learn region correlations. The reasons are twofold. First, RAM evaluates the correlation between any pair of regions instead of only focusing on nearby ones. As a result, it does not rely on stacking too many layers~($\gg k$) to consider the correlation with distant regions. Moreover, the evaluation process could be completed in parallel. Secondly, surrounding observation preserves plenty of information for traffic forecasting. With more information provided, the convergence of the model would be improved, and the training process is hence speeded up.

\subsection{Dependency Learning over Multiple Time Slots~(DLM)}
\label{sec:dependency learning over multiple time slots}

Considering the spatial-temporal dependency may have periorical characteristic, DLM learns the periodic dependency by evaluating the correlation between $\mathcal{R}_n^{t}$ and the dependency at other historical time slots $\mathcal{R}_i^{\hat{t}}$~(where $\hat{t} < t$). After that, it generates a new embedding for $r_i$ by aggregating the embedding at different time slots according to their correlations.

Specifically, we consider the short-term and long-term period for traffic data in this work. The spatial-temporal dependency at the following historical time slots are used as the model input to predict the inflow and outflow of regions at $t$: spatial-temporal dependency in the recent $h$ time slots~(i.e., short-term period); the same time interval in the recent $l$ days~(i.e.,long-term period).

We employ a point-wise aggregation with self-attention to evaluate their correlations and aggregate the embedding accordingly. To capture the multiple periodic dependency, we use an multi-head attention network in DLM.

Given a set of historical time slot $Q$, the $z$-th dependency on a historical time slot $\hat{t} \in Q$  is calculated as follows:
\begin{equation}
    e_z(t, \hat{t}) = \frac{(\mathcal{R}_i^t \cdot W_{Q_z}^\mathcal{T}) \cdot ( \mathcal{R}_{i}^{\hat{t}} \cdot W_{K_z}^\mathcal{T})^T}{\sqrt{d}}, 
\end{equation}
where $W_{Q_z}^\mathcal{T} \in \mathbb{R}^{d\times d}$ and $W_{K_z}^\mathcal{T} \in \mathbb{R}^{d \times d}$ are learnable parameters. 

We use a softmax function to normalize the dependency and aggregate the context of each time slot by weight:
\begin{equation}
    \beta_z(t, \hat{t}) = softmax(e_z(t, \hat{t})) = \frac{\exp(e_z(t,\hat{t}))}{\sum_{t_p \in Q} \exp(e_z(t,t_p))}. 
\end{equation}

The aggregation of the $z$-th head is hence 
\begin{equation}
    \hat{\mathcal{R}}_{i,z}^{t} = (\sum_{\hat{t} \in Q} (\beta_z(t,\hat{t}) \cdot \mathcal{R}_{i}^{\hat{t}}) \cdot W_V^\mathcal{T},
\end{equation}
where $Q$ is the set of historical time slots, $\hat{\mathcal{R}}_{i,z}^{t}$ is the aggregation result of the $z$-th head, and $W_V^\mathcal{T} \in \mathbb{R}^{d \times d}$ are learnable parameters. Finally, we concatenate the results of different heads with

\begin{equation}
    \hat{\mathcal{R}}_{i}^{t} = Concat(\hat{\mathcal{R}}_{i,1}^{t}, \hat{\mathcal{R}}_{i,2}^{t}, \cdots, \hat{\mathcal{R}}_{i,Z}^{t}) \cdot W_\mathcal{T} ,
\end{equation}
where  $Concat(\cdot)$ is the concatenation operation, and $W_\mathcal{T} \in \mathbb{R}^{(Z \times d) \times d}$ are learnable parameters. We then pass $\hat{\mathcal{R}}_{i}^{t}$ to a fully connected neural network to obtain the spatial-temporal embedding of $r_i$ at $t$. Note that we also employ a residual connection between each of the two layers to avoid gradient exploding or vanishing. The output of the DLM is denoted as $\Omega_{i}^{t}$ for region $r_i$ at $t$.

Different from prior works, the periodic dependency learning is conditional on the spatial-temporal dependency at each individual time slot. Consequently, the spatial and temproal dependencies are jointly considered during the periodic dependency learning. The computation complexity is $O(|Q|)$, where $|Q|$ is the number of historical time slots used for learning.
%Because TAM evaluates the correlation between the embedding at $t$~$(\mathcal{R}^t_i)$ and at other historical time slots, the computation complexity is $O(q)$ rather than $O(q^2)$ of a conventional Transformer, where $q$ is the number of historical time slots. Furthermore, by comparing $\mathcal{R}^t_i$ with other historical time slots and aggregating them regarding their dependency, the irregularity effect at $t$ could be considered, and the temporal dependency could be captured in the embedding. 

\subsection{Prediction Network~(PN)}
\label{sec:forecasting_layer}
Given the spatial-temporal embedding $\mathcal{T}_{r_i}^{t}$ of a region, the prediction network predicts the inflow and outflow using the fully connected network. The forecasting function is defined as 
\begin{equation}
    [\hat{\mathcal{I}}^t_{i}, \hat{\mathcal{O}}^t_{i}] = \sigma(\Omega_{i}^{t} \cdot W^P + b^P),
\end{equation}
where $\hat{\mathcal{I}}^t_{i}$ and $\hat{\mathcal{O}}^t_{i}$ is the forecasting inflow and outflow respectively, $\sigma(\cdot)$ is the ReLU activation function and $W^P \in \mathbb{R}^{d \times 2}$, and $b^P \in \mathbb{R}^{1 \times 2}$ are learnable parameters.

%\subsection{Loss Function}
We simultaneously forecast the inflow and outflow in our work, and define the loss function as follows:
\begin{equation}
    %LOSS = \frac{1}{2n}\sqrt{ \sum_{i=1}^{n}(I^i_{t} - \hat{I}^i_{t})^2 + \sum_{i=1}^{n}(O^i_{t} - \hat{O}^i_{t})^2},
    LOSS = \sqrt{\frac{\sum_{i=1}^{n}(\mathcal{I}^t_{i} - \hat{\mathcal{I}}^t_{i})^2 + \sum_{i=1}^{n}(\mathcal{O}^t_{i} - \hat{\mathcal{O}}^t_{i})^2}{2n}},
\end{equation}
where $n$ is the number of regions.
%and $\hat{s}^i_{t+1}$ and $\hat{e}^i_{t+1}$ are the forecast start and end traffic volume, respectively.

%\subsection{Algorithm}\label{algo}

%\input{ETF.tex}

%We illustrate the algorithm of \FRAMEWORK{} in Algorithm \ref{algorithm:ETF}. The time slot duration is set as $30$ minutes. Given the inflow and outflow data of all regions from time slots $0$ to $t-1$ and temporal parameters $h$, $d$, and $y$,  it first employs RAM to calculate the embedding of regions at the past $h$ time slots~(Lines $3$ - $5$), at the same time-of-day in the past $d$ days~(Line $6$ - $8$), and at the same time-of-week on the same day-of-week~(Line $9$ - $11$). Then, for each region $r_i$, \FRAMEWORK{} uses the TAM to consider the temporal dependency and irregularity effect, obtaining the spatial-temporal embedding $\mathcal{T}_i^t$ for $s_i$~(Line $12$ - $14$). At last, the spatial-temporal embedding of all regions is passed to the prediction network for prediction~(Line $15$). Note that the computation of the above algorithm can be completed in parallel and be speeded up using GPU.

\section{Illustrative Experimental Results}
\label{sec:evaluation}
In this section, we first introduce the datasets and the data processing approaches in Section~\ref{sec:dataset}, and the evaluation metrics and baseline approaches in Section~\ref{sec:metric}. Then, we compare the accuracy and training efficiency of \modelname{} with the state-of-the-art methods in Sections~\ref{sec:effectiveness} and \ref{sec:efficiency}, respectively.
After that, we evaluate the performance of variants of \modelname{} and the  effect of  surrounding observations in Sections~\ref{sec:Surrounding Observation} and~\ref{sec:variants}, respectively, followed by the discussion on the hyperparameters
of layer number in Section~\ref{sec:layer_num} and head number in Section~\ref{sec:head_num}.

\subsection{Datasets}
\label{sec:dataset}

We conduct extensive traffic study and model evaluations based on two real-world traffic flow datasets collected in New York City~(NYC), the NYC-Taxi dataset and the NYC-Bike dataset. Each dataset contains trip records, each of which consists of origin, destination, departure time, and arrival time. The NYC-Taxi dataset contains $22,349,490$ taxi trip records of NYC in 2015, from 01/01/2015 to 03/01/2015. The NYC-Bike dataset was collected from the NYC Citi Bike system from 07/01/2016
to 08/29/2016, and contains $2,605,648$ trip records. 

The city is split into $10 \times 20$ regions with a size of $1$km$\times1$km. The time interval is set as $30$ minutes for both datasets. 
%Although transition data between regions are available in the two datasets, we do not use them for model training and only use them for the case study in Section \ref{sec:case_study} so that our solution can also be applied in scenarios where transition data are not available. 
The two datasets were pre-processed and released online\footnote{https://github.com/tangxianfeng/STDN/blob/master/data.zip} by the prior work~\cite{yao2019revisiting}. 

%Following the prior work~\cite{yao2019revisiting}, 
In both of the taxi dataset and bike dataset, we use data from the previous $40$ days as the training data, and the remaining $20$ days as the testing data. In the training data, we select $80\%$ of the training data to train our model and the remaining $20\%$ for validation. We use the Min-Max normalization to rescale the range of volume value in $[0, 1]$, and recover the result for evaluation after forecasting. In our experiments, we exclude the results of those regions the inflow or outflow of which is less than 10 when evaluating the model. It is a common practice used in industry and many prior works~\cite{yao2018deep,yao2019revisiting,li2019forecaster}. 

\subsection{Performance Metrics and Baseline Methods}
\label{sec:metric}
We use Root Mean Squared Errors~(RMSE) and Mean Average Percentage Error~(MAPE) as the evaluation metrics, which are defined as follows:
\begin{equation}
    RMSE =  \sqrt{\frac{\sum_{i=1}^{N} (y_i - \hat{y}_i)^2}{N}},
\end{equation}
\begin{equation}
    MAPE = \frac{1}{N} \sum_{i=1}^{N}\frac{|y_i - \hat{y}_i|}{y_i},
\end{equation}
where $y_i$ and $\hat{y}_i$ are the ground-truth and forecasting result of the $i$-th sample, and $N$ is the total number of samples.

%We compare our model with the state-of-the-art approaches, including the conventional time series data forecasting approaches: (1) Historical average~(HA); and (2) Autoregressive integrated moving average~(ARIMA); conventional machine learning approaches: (3)~Ridge Regression~(Ridge); %(4)~Lin-UOTD; 
%(4)~XGBoost~\cite{chen2016xgboost}; and (5)~Multi-Layer Perceptron~(MLP); and some state-of-the-art deep learning models for traffic forecasting: (6)~Convolutional LSTM~(ConvLSTM)~\cite{xingjian2015convolutional}; %(8)~DeepSD~\cite{wang2017deepsd}; 
%(7)~Deep Spatio-Temporal Residual Networks~(ST-ResNet)~\cite{zhang2017deep}; (8)~Deep Multi-View Spatial-Temporal Network~(DMVST-Net)~\cite{yao2018deep}; (9)~Spatial-Temporal Dynamic Network~(STDN)~\cite{yao2019revisiting}; (10)~Spatiotemporal attention mechanism-based model~(ST-Attn)~\cite{zhou2020spatiotemporal}; (11)~Spatial-Temporal Self-Attention Network~(STSAN)~\cite{lin2020interpretable}; and (12)~ Dynamic Switch-Attention Network~(DSAN)~\cite{lin2020preserving}. More details of the baselines are discussed in Appendix \ref{app:baseline}, and the default hyperparameters of our \modelname{} are presented in Appendix \ref{app:hyperparameters}.

\begin{table*}%[htb]
    \caption{Comparison with the state-of-the-art methods.}
    \label{table: comparison_with_baselines}
    \centering
    \begin{tabularx}{0.8\textwidth}{ 
      | >{\centering\arraybackslash}X 
      || >{\centering\arraybackslash}X 
      | >{\centering\arraybackslash}X 
      | >{\centering\arraybackslash}X 
      | >{\centering\arraybackslash}X 
      | >{\centering\arraybackslash}X | }
    %{c||c|c|c|c|c}
    \hline
    \multirow{2}*{Dataset}&\multirow{2}*{Method}&\multicolumn{2}{c|}{Inflow}&\multicolumn{2}{c|}{Outflow}\\
    \cline{3-6}
     &  & RMSE  & MAPE & RMSE & MAPE  \\
    \hline
    \multirow{13}*{NYC-Taxi}&{HA}&{33.83}&{21.14\%}&{43.82}&{23.18\%}\\
    &{ARIMA}&{27.25}&{20.91\%}&{36.53}&{22.21\%}\\
    &{Ridge}&{24.38}&{20.07\%}&{28.51}&{19.94\%}\\
    %&{LinUOTD}&{24.39}&{20.03\%}&{28.48}&{19.91\%}\\
    &{XGBoost}&{21.72}&{18.70\%}&{26.07}&{19.35\%}\\
    &{MLP}&{22.08$\pm$0.50}&{18.31$\pm$0.83\%}&{26.67$\pm$0.56}&{18.43$\pm$0.62\%}\\
    &{ConvLSTM}&{23.67$\pm$0.20}&{20.70$\pm$0.20\%}&{28.13$\pm$0.25}&{20.50$\pm$0.10\%}\\
    %&{DeepSD}&{21.95$\pm$0.35}&{18.15$\pm$0.62\%}&{26.35$\pm$0.53}&{18.12$\pm$0.38\%}\\
    &{ST-ResNet}&{21.63$\pm$0.25}&{21.09$\pm$0.51\%}&{26.23$\pm$0.33}&{21.13$\pm$0.63\%}\\
    %&{DMVST-Net}&{20.51$\pm$0.46}&{17.14$\pm$0.32\%}&{25.74$\pm$0.26}&{17.38$\pm$0.46\%}\\
    &{STDN}&{19.05$\pm$0.31}&{16.25$\pm$0.26\%}&{24.10$\pm$0.25}&{16.30$\pm$0.23\%}\\
    &{ASTGCN}&{22.05$\pm$0.37}&{20.25$\pm$0.26\%}&{26.10$\pm$0.25}&{20.30$\pm$0.31\%}\\
    &{STGODE}&{21.46$\pm$0.42}&{19.22$\pm$0.36\%}&{27.24$\pm$0.46}&{19.30$\pm$0.34\%}\\
    %&{ST-ATTN}&{28.53$\pm$1.02}&{26.70$\pm$1.62\%}&{34.41$\pm$1.04}&{25.87$\pm$1.67\%}\\
    &{STSAN}&{23.07$\pm$0.64}&{22.24$\pm$1.91\%}&{27.83$\pm$0.30}&{25.90$\pm$1.67\%}\\
    &{DSAN}&{18.32$\pm$0.39}&{16.07$\pm$0.31\%}&{24.27$\pm$0.30}&{17.70$\pm$0.35\%}\\
    \cline{2-6}
    %&\textbf{STT-REST}&{\textbf{17.98$\pm$0.18}}&{\textbf{14.63$\pm$0.36\%}}&{\textbf{22.78$\pm$0.22}}&{\textbf{14.73$\pm$0.38\%}}\\
    &\textbf{ST-TIS}&{\textbf{17.73$\pm$0.23}}&{\textbf{14.65$\pm$0.32\%}}&{\textbf{21.96$\pm$0.13}}&{\textbf{14.83$\pm$0.76\%}}\\
    
    \hline
    \hline
    \multirow{13}*{NYC-Bike}&{HA}&{11.93}&{27.06\%}&{12.49}&{27.82\%}\\
    &{ARIMA}&{11.25}&{25.79\%}&{11.53}&{26.35\%}\\
    &{Ridge}&{10.33}&{24.58\%}&{10.92}&{25.29\%}\\
    %&{LinUOTD}&{8.94}&{24.44\%}&{9.57}&{23.52\%}\\
    &{XGBoost}&{8.94}&{22.54\%}&{9.57}&{23.52\%}\\
    &{MLP}&{9.12$\pm$0.24}&{22.40$\pm$0.40\%}&{9.83$\pm$0.19}&{23.12$\pm$0.24\%}\\
    &{ConvLSTM}&{9.22$\pm$0.19}&{23.20$\pm$0.47\%}&{10.40$\pm$0.17}&{25.10$\pm$0.45\%}\\
    %&{DeepSD}&{9.08}&{22.36\%}&{9.69}&{23.62\%}\\
    &{ST-ResNet}&{8.85$\pm$0.13}&{22.98$\pm$0.53\%}&{9.80$\pm$0.12}&{25.06$\pm$0.36\%}\\
    %&{DMVST-Net}&{8.50$\pm$0.19}&{21.56$\pm$0.49\%}&{9.14$\pm$0.13}&{22.20$\pm$0.33\%}\\
    &{STDN}&{8.15$\pm$0.15}&{20.87$\pm$0.39\%}&{8.85$\pm$0.11}&{21.84$\pm$0.36\%}\\
    &{ASTGCN}&{9.05$\pm$0.31}&{22.25$\pm$0.36\%}&{9.34$\pm$0.24}&{23.13$\pm$0.30\%}\\
    &{STGODE}&{8.58$\pm$0.38}&{23.33$\pm$0.26\%}&{9.23$\pm$0.31}&{23.99$\pm$0.23\%}\\
    %&{ST-ATTN}&{9.66$\pm$0.36}&{25.12$\pm$0.86\%}&{10.30$\pm$0.26}&{26.00$\pm$0.5\%}\\
    &{STSAN}&{8.20$\pm$0.45}&{20.42$\pm$1.33\%}&{9.87$\pm$0.23}&{23.87$\pm$0.71\%}\\
    &{DSAN}&{7.97$\pm$0.25}&{20.23$\pm$0.18\%}&{10.07$\pm$0.58}&{23.92$\pm$0.39\%}\\
    \cline{2-6}
    %&\textbf{STT-REST}&{\textbf{7.68$\pm$0.19}}&{\textbf{19.98$\pm$0.18\%}}&{\textbf{8.20$\pm$0.13}}&{\textbf{20.71$\pm$0.20\%}}\\
    &\textbf{ST-TIS}&{\textbf{7.57$\pm$0.04}}&{\textbf{18.64$\pm$0.23\%}}&{\textbf{7.73$\pm$0.10}}&{\textbf{18.58$\pm$0.19\%}}\\
    \hline
    \end{tabularx}
    \end{table*}
    
We compare our model with the following state-of-the-art approaches: 
\begin{itemize}
    \item {\em Historical average~(HA):} It uses the average of traffic at the same time slots in historical data for prediction.
    \item {\em ARIMA:} It is a conventional approach for time series data forecasting.
    \item {\em Ridge Regression:} A regression approach for time series data forecasting.
    \item {\em XGBoost~\cite{chen2016xgboost}:} A powerful approach for building supervised regression models.
    \item {\em Multi-Layer Perceptron~(MLP):} A three-layer fully-connected
    neural network.
    \item {\em Convolutional LSTM~(ConvLSTM)~\cite{xingjian2015convolutional}:} It is a special recurrent neural network with a convolution structure for spatial-temporal prediction.
    \item {\em ST-ResNet~\cite{zhang2017deep}:} It uses multiple convolutional networks with residual structures to capture spatial correlations from different temporal periods for traffic forecasting. It also considers external data such as weather, holiday events, and metadata.
    %\item (8) DMVST-Net~\cite{yao2018deep}: It models spatial correlation using local CNN and temporal correlations using LSTM for taxi demand prediction.
    \item {\em STDN~\cite{yao2019revisiting}:} It considers the dynamic spatial correlation and temporal shifted problem using the combination of CNN and LSTM. External data such as weather and event are considered in the work.
    \item {\em ASTGCN~\cite{guo2019attention}:} It is an attention-based spatial-temporal graph convolutional network (ASTGCN) model to solve traffic flow forecasting problem.
    %\item {\em  ST-Attn~\cite{zhou2020spatiotemporal}:} It is a general encoder-decoder framework based on self-attention for modeling sequential data without using recurrent neural network units. 
    \item {\em STGODE~\cite{fang2021spatial}:} It uses a spatial-temporal graph ordinary differential equation network to predict traffic flow based on two predefine graph, namely a spatial graph in terms of distance, and a semantic graph in terms of flow similarity.
    \item {\em STSAN~\cite{lin2020interpretable}:} It uses CNN to capture spatial information and the canonical Transformer to consider the temporal dependencies over time. In particular, transition data between regions are used to indicate the correlation between regions. 
    \item {\em DSAN~\cite{lin2020preserving}:} It uses the canonical Transformer to capture the spatial-temporal correlations for spatial-temporal prediction, in which transition data between regions are used for correlation modeling.
\end{itemize}

We use the identical datasets and data process approach as the work STDN~\cite{yao2019revisiting}, and use the results of the work~\cite{yao2019revisiting} as the benchmark for discussion. The experiment results of (1) $\sim$ (8) in Table \ref{table: comparison_with_baselines} are reported in the work~\cite{yao2019revisiting}. 
The evaluation of ASTGCN, STGODE, STSAN, and DSAN is based on the code from their authors' GitHubs.

We implement our model using PyTorch. Data and code can be found in https://github.com/GuanyaoLI/ST-TIS. We use the following data as the model input since they achieve the best performance on the validation datasets: data in the recent past 3 hours~(i.e., $h = 6$ as the slot duration is $30$ minutes); the same time slot in the recent past 10 days~(i.e., $l = 10$). 
The other default hyperparameter settings are as follows. 
%The default hyperparameter settings are as follows: 
%The temporal parameters $h$, $d$, and $y$ are set as $6$, $6$, and $4$ respectively since they achieve the best performance on the vilidation dataset. 
The default length $w$ of the surrounding observations is $6$~(i.e. $3$ hours), the number of convolution kernels $F$ is $4$, and the dimension $d$ is set as $8$. The number of $k$ for DLI in Figure \ref{fig:DLI} is set as $3$, and the number of heads is set as $6$ for the two modules. 
Furthermore, the dropout rate is set as $0.1$, the learning rate is set as $0.001$, and the batch size is set to be $32$. Adam optimizer is used for model training. We trained our model on a machine with a NVIDIA RTX2080 Ti GPU.

%We use the following data as the model input to predict the inflow and outflow of regions at $t$: data in the recent past 3 hours~(i.e. short-term): ${t-i}$, where $i = 1,2,...,6$; the same time interval in the recent past 6 days~(i.e. daily periodic): ${t-j\times48}$, where $j = 1, 2, ... 6$; and the same time interval on the same day of the past $4$ weeks~(i.e. weekly periodic): $t - y\times48\times7$, where $y = 1, ... , 4$. 

\subsection{Prediction Accuracy}
\label{sec:effectiveness}

We compared the accuracy of \modelname{} with the state-of-the-art methods using the metrics RMSE and MAPE. The results for the two datasets are presented in Table \ref{table: comparison_with_baselines}. Each approach was run $10$ times, and the mean and standard deviation are reported. As shown in the table, \modelname{} significantly outperforms all other approaches on both metrics and datasets.  

Specifically, the performance of the conventional time series forecasting approaches~(HA and ARIMA) is poor for both datasets because these approaches do not consider the spatial dependency. Conventional machine learning approaches~(Ridge, XGBoost, and MLP), which consider spatial dependency as features, have better performance than HA and ARIMA. However, they fail to consider the joint spatial-temporal dependencies between regions. Most deep learning-based models have further improvements than conventional works, illustrating the ability of deep neural networks to capture the complicated spatial and temporal dependency. 
%However, the performance of ST-Attn is not good in our experiments, which is inconsistent with the results reported in the work~\cite{zhou2020spatiotemporal}. The potential reason is that ST-Attn relies heavily on a huge amount of data for training. More than two years of data are used for training in the experiment of the work~\cite{zhou2020spatiotemporal}, but only $40$ days of data are provided in our experiments. 
%
%\modelname{} is substantially more accurate than CNN-based and GCN-based models~(i.e., ConvLSTM, STResNet, STDN, ASTGCN and STGODE) 
\modelname{} is substantially more accurate than state-of-the-art approaches~(i.e., ConvLSTM, STResNet, STDN, ASTGCN and STGODE). For example, it has an average improvement of $9.5\%$ on RMSE and $12.4\%$ on MAPE compared to STDN and DSAN.
The reasons for the improvements are that it can capture the correlations between both nearby and distant regions, and it considers the spatial and temporal dependency in a joint manner. We find that the improvement is more significant on the NYC-Taxi dataset than on the NYC-Bike dataset. The reason could be that people prefer using taxis instead of bikes for long-distance travel, and so the correlations with distant regions are more important for the prediction task on the taxi dataset.  The significant improvement demonstrates that \modelname{} has a better ability to capture the correlations for distant regions than the other approaches which have the spatial locality assumption. \modelname{} also outperforms other Transformer-based approaches~(such as STSAN and DSAN). The reason is that with the information fusion and region sampling strategies in \modelname{}, the long-tail issue of the canonical Transformer is addressed and the joint spatial-temporal correlations are considered. The significant improvements demonstrate the effectiveness of our proposed model.

\subsection{Training Efficiency}
\label{sec:efficiency}

\begin{table}%[htbp]
\caption{Comparison of training time.}
\label{table: comparison_of_running_time}
\centering
\begin{tabular}{|c||c|c|c|}
\hline
                        Dataset  & Method  & \begin{tabular}[c]{@{}c@{}}Average time\\  per epoch~(s)\end{tabular} & Total time~(s) \\ \hline
\multirow{7}{*}{NYC-Taxi} & ST-ResNet & 7.31     & 3077.51    \\ \cline{2-4} 
%& DMVST-Net & - & - \\ \cline{2-4}
& STDN    & 445.47            & 34746.66   \\ \cline{2-4}  
& ASTGCN & 25.31 & 6272.88 \\ \cline{2-4}  
& STGODE & 18.53 & 3423.48 \\ \cline{2-4}                      
& STSAN & 426.75  & 33769.32   \\ \cline{2-4}
& DSAN & 386.17  & 29390.75    \\ \cline{2-4}
& \textbf{ST-TIS}  & {10.21}  & \textbf{1231.5}    \\ \hline
\multirow{7}{*}{NYC-Bike} & ST-ResNet & 7.25   & 2921.75     \\ \cline{2-4}
%& DMVST-Net & - & - \\ \cline{2-4}
& STDN  & 480.21  & 23066.43      \\ \cline{2-4}
& ASTGCN & 25.68 & 5084.64 \\ \cline{2-4}   
& STGODE & 18.76 &  3752.89\\ \cline{2-4}                   
& STSAN & 426.28  & 31216.28    \\ \cline{2-4} 
& DSAN & 434.03   & 26476.20    \\ \cline{2-4}                      
& \textbf{ST-TIS}  & {10.37}  & \textbf{1556.8}  \\ \hline
\end{tabular}
\end{table}

Training and deploying large deep learning models is costly.
For example, the cost of trying combinations of different hyper-parameters for a large model is computationally expensive and it highly relies on training resources~\cite{menghani2021efficient}.
Thus, we compare the training efficiency of \modelname{} with some state-of-the-art deep learning based approaches~(i.e., ST-ResNet, STDN, ASTGCN, STGODE, STSAN, and DSAN) in terms of training time and number of learnable parameters.

%STDN and DSAN are selected because they have better performance than other baseline approaches in terms of RMSE and MAPE. STSAN is also selected for comparison because it also uses self-attention to capture the spatial and temporal correlations. 
The results of the average training time per epoch and the total training time are presented in Table \ref{table: comparison_of_running_time}. 
\begin{table}%[htbp]
    \caption{Comparison of the number of parameters.}
    \label{table: comparison_of_number_of_parameters}
    \centering
    \begin{tabular}{|c|c|}
    \hline
    Method  & Number of parameters \\ \hline
    ST-ResNet & 4,917,041 \\ \hline
    %DMVST-Net & - \\ \hline
    STDN    & 9,446,274            \\ \hline
    ASTGCN & 450,031 \\ \hline
    STGODE &  433,073\\ \hline
    %ST-Attn & 384,716              \\ \hline
    DSAN  & 1,621,634           \\ \hline
    STSAN  & 34,327,298           \\ \hline
    \textbf{ST-TIS}  & \textbf{139,506} \\ \hline
    \end{tabular}
    \end{table}
ST-ResNet achieves the least training time among all comparison approaches because it solely employs simple CNN and does not rely on RNN for temporal dependency learning. The average time per epoch of \modelname{} is close to ST-ResNet. In addition, \modelname{} is trained significantly faster than other approaches~(with a reduction of $46\% \sim 95\%$).
%because it entirely relies on the self-attention mechanism to capture the spatial and temporal correlations, which can be trained in parallel. 
STDN uses LSTM to capture temporal correlation, which
%As LSTM generates a sequence of hidden states as a function of the previous hidden state and the input for the position, it 
is in general more difficult to be trained in parallel. STSAN and DSAN also employ Transformer with self-attention for spatial and temporal correlation learning, but our proposed approach is significantly efficient than them, illustrating the efficiency of the proposed region graph for model training. 
%both of them rely on transition data between any pair of regions, leading to a complicated model and more training time. 

Furthermore, we also compare the number of learnable parameters of each model in Table \ref{table: comparison_of_number_of_parameters}. More parameters may lead to difficulties in model training, and it requires more memory and training resources. Compared with other state-of-the-art approaches, \modelname{} is much more lightweight with fewer parameters for training~(with a reduction of $23\% \sim 98\%$). The comparison results in Tables \ref{table: comparison_with_baselines}, \ref{table: comparison_of_running_time} and \ref{table: comparison_of_number_of_parameters} demonstrate that our proposed \modelname{} is faster and more lightweight than other deep learning based baseline approaches, while achieving even better prediction accuracy.

\subsection{Design Variations of \modelname{}}
\label{sec:variants}

\begin{figure}%[ht]
    \noindent\makebox[\linewidth][c]{
        \subfigure[\small RMSE.]{
            \label{fig:taxi_variant_rmse}
            \includegraphics[width=0.24\textwidth]{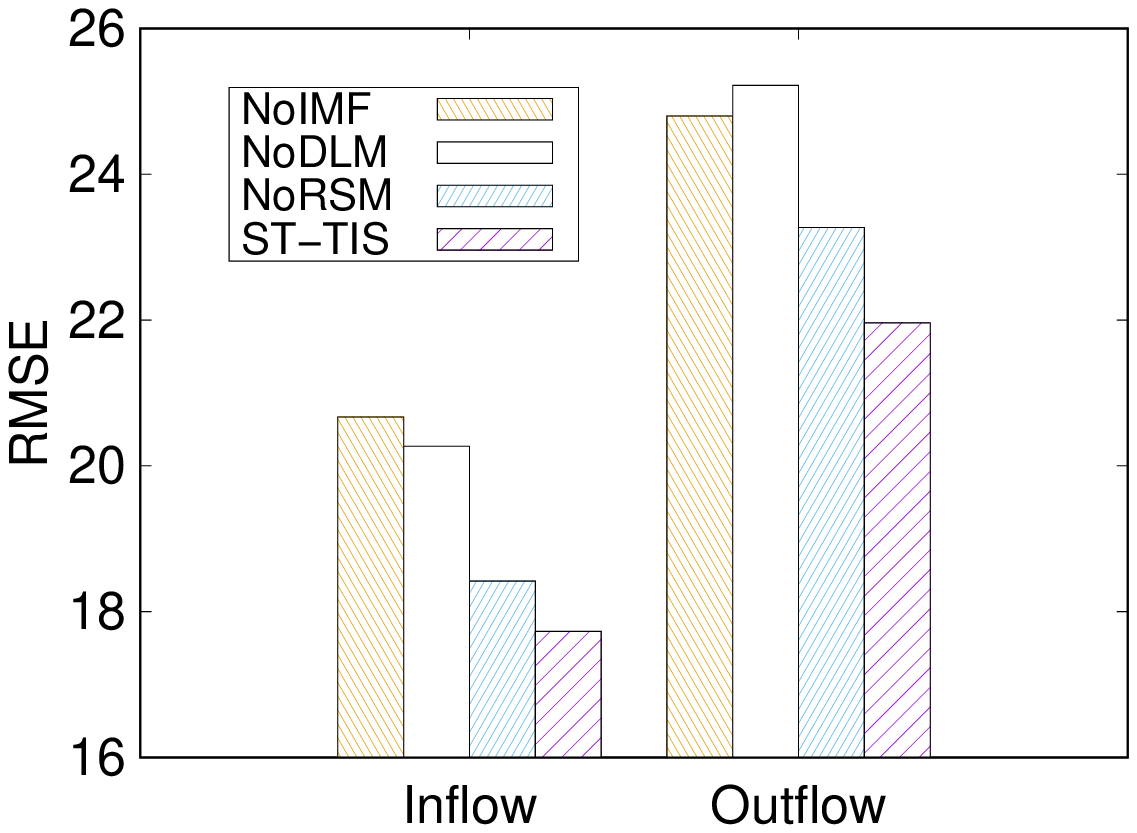}
        }
        %\vspace{.6in}
        \subfigure[\small MAPE.]{
            \label{fig:taxi_variant_mape}
            \includegraphics[width=0.24\textwidth]{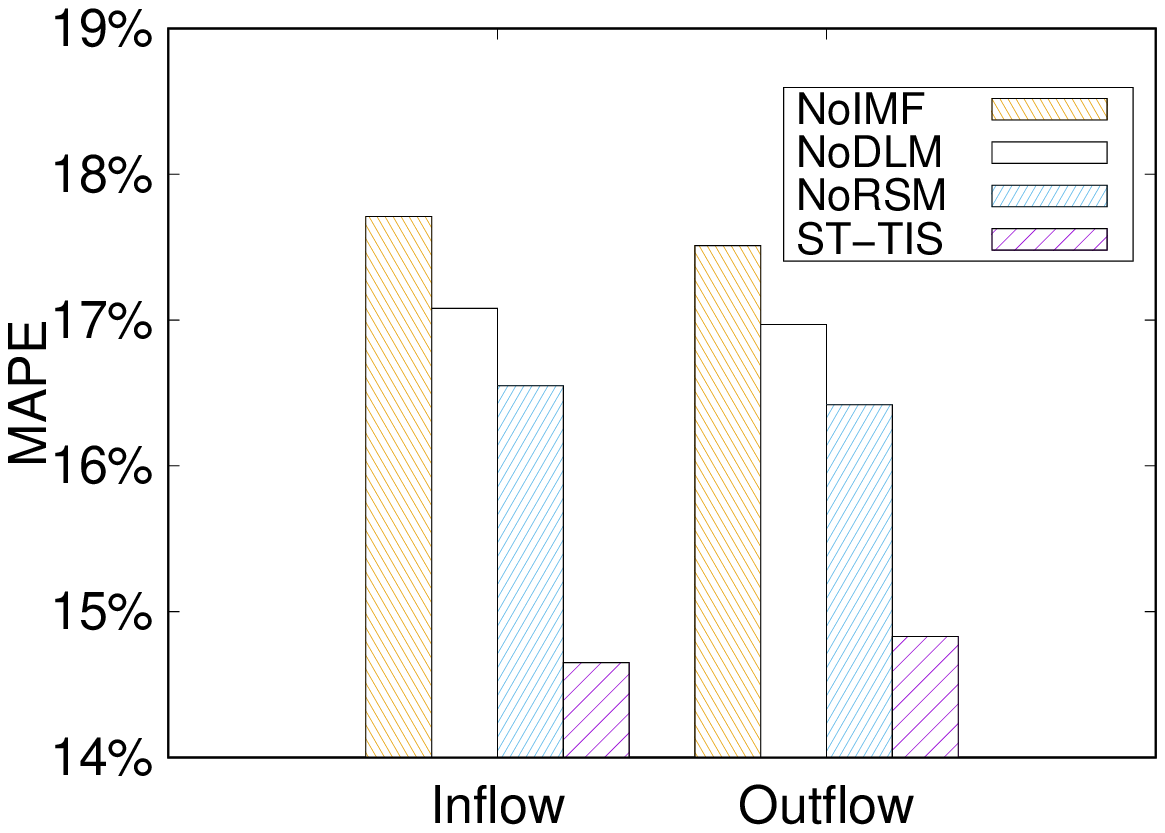}
        }
        
    }
    \caption{Performance of variants on the NYC-Taxi dataset.}
    \end{figure}
    
    \begin{figure}%[ht]
    \noindent\makebox[\linewidth][c]{
        \subfigure[\small RMSE.]{
            \label{fig:bike_variant_rmse}
            \includegraphics[width=0.24\textwidth]{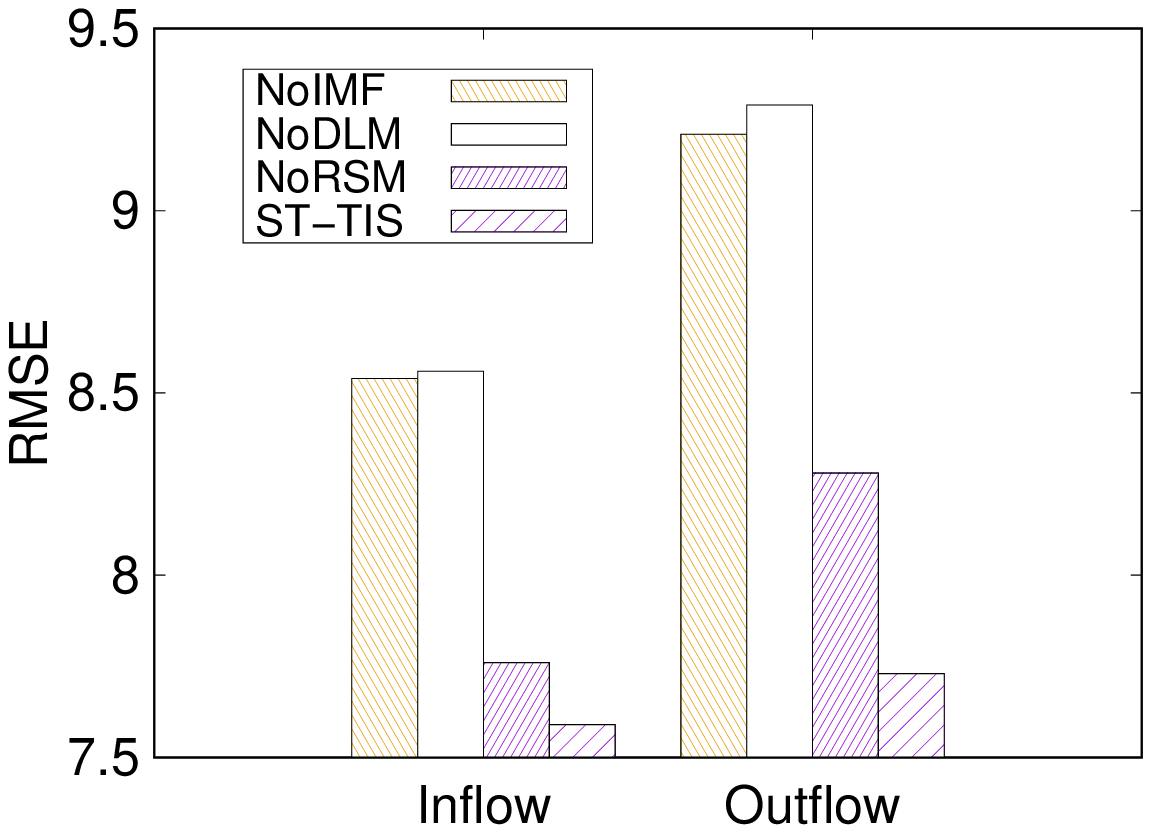}
        }
        %\vspace{.6in}
        \subfigure[\small MAPE.]{
            \label{fig:bike_variant_mape}
            \includegraphics[width=0.24\textwidth]{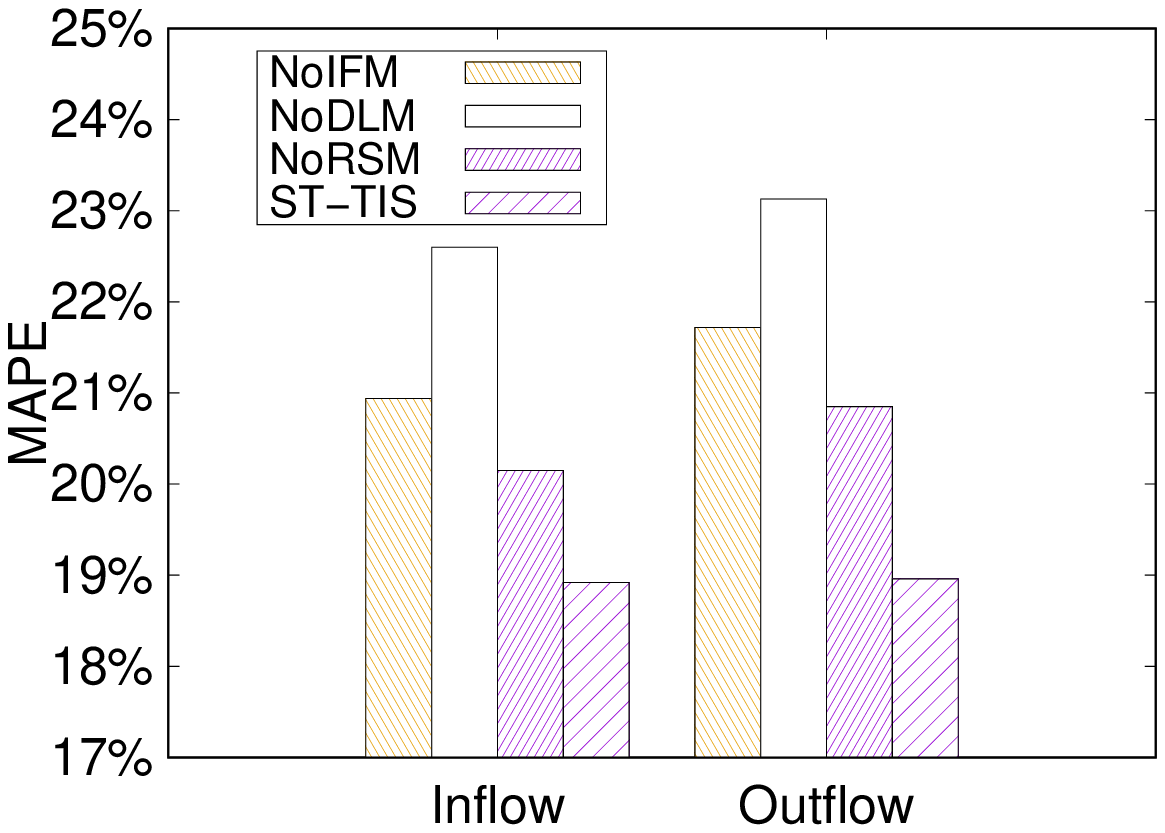}
        }
        
    }
    \caption{Performance of variants on the NYC-Bike dataset.}
    \end{figure}

We compare \modelname{} with its variants to evaluate the effectiveness of the proposed modules. The following variants are discussed: 
 \begin{itemize}
    \item \textit{NoIFM}: We remove the information fusion module from \modelname{}. Only the surrounding observation of a time slot is used as the input of the model instead of the fusion result.
    \item \textit{NoRSM}: We remove the region sampling module from \modelname{}. The canonical Transformer is used to capture the dependencies between regions without region sampling.
    \item \textit{NoDLM}: We remove the module of dependency learning over multiple time slots, and only use $\mathcal{R}^t_i$ as the input of the prediction network for prediction.
    %\item \textit{Short-term}: We only consider the dependency in the recent past 3 hours~(i.e., short-term period) for the DLM.
    %\item \textit{Long-term}: We consider dependency at the same time slot in the recent past 6 days~(i.e., long-term peroid) for the DLM.
 \end{itemize}

The RMSE and MAPE on the NYC-Taxi dataset are presented in Figures \ref{fig:taxi_variant_rmse} and \ref{fig:taxi_variant_mape}, respectively. After taking the information fusion module away, the performance degrades significantly. The reason is that the information fusion module plays a fundermental role in jointly considering the spatial-temporal dependency. Without such module, our approach would degenerate to consider the spatial and temporal dependency in a decouple manner. The experimental results demonstrate the importance of considering spatial-temporal dependency jointly and the effective of the proposed information fusion module. 
Moreover, without the region sampling module, our approach still achieve a good prediction performance because the canonical Transformer is good at capturing dependencies between regions. The performance is further improved with the region sampling since it could address the long-tail issue of the canonical Transformer for embedding aggregation. Furthermore, the RMSE and MAPE increase when the periodical characteristic of spatial-temporal dependency is not considered~(i.e., NoDLM), which indicates the necessity of considering the period of spatial-temporal dependency and demonstrates the effectiveness and rationality of our model design.
Similar and consistent findings can be observed on the NYC-Bike dataset in Figures \ref{fig:bike_variant_rmse} and \ref{fig:bike_variant_mape}.

\subsection{Surrounding Observations}
\label{sec:Surrounding Observation}

\begin{figure}%[ht]
    \noindent\makebox[\linewidth][c]{
        \subfigure[\small RMSE.]{
            \label{fig:taxi_spatial_len_rmse}
            \includegraphics[width=0.24\textwidth]{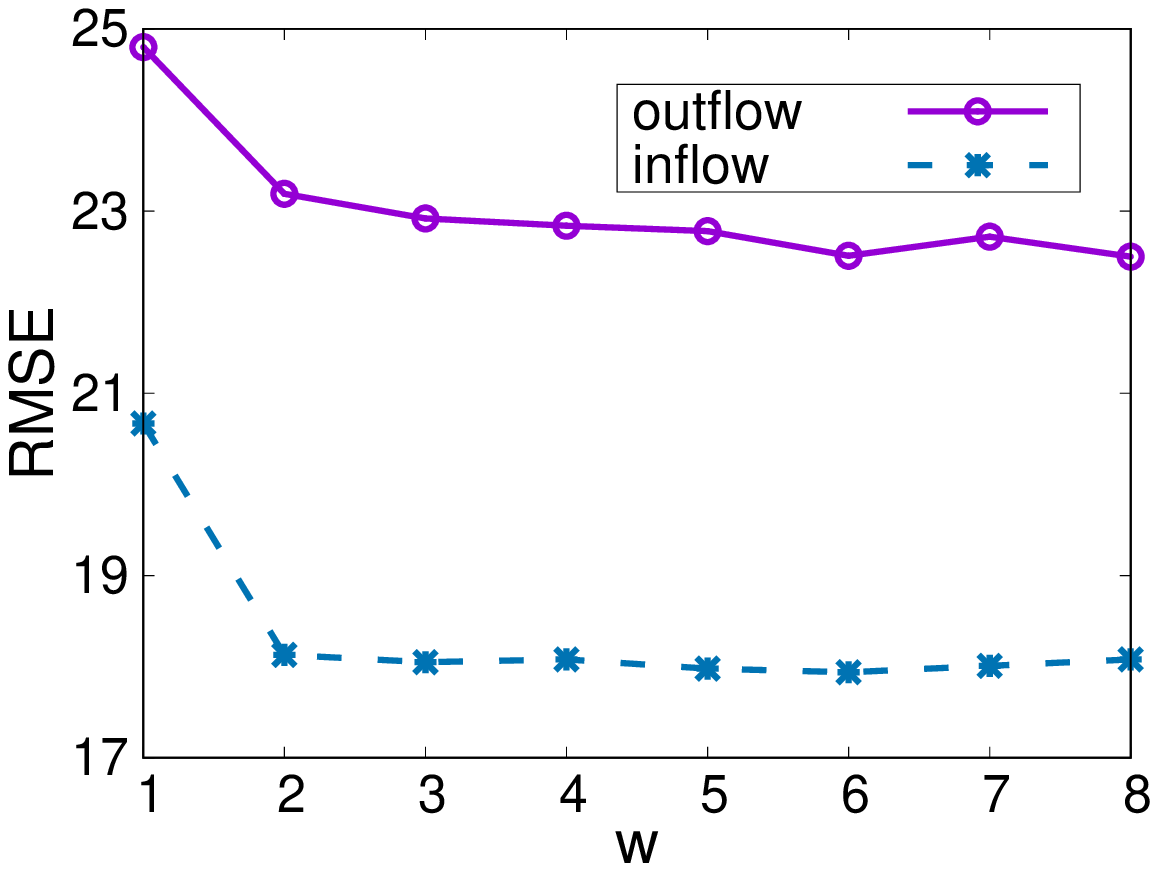}
        }
        %\vspace{.6in}
        \subfigure[\small MAPE.]{
            \label{fig:taxi_spatial_len_mape}
            \includegraphics[width=0.24\textwidth]{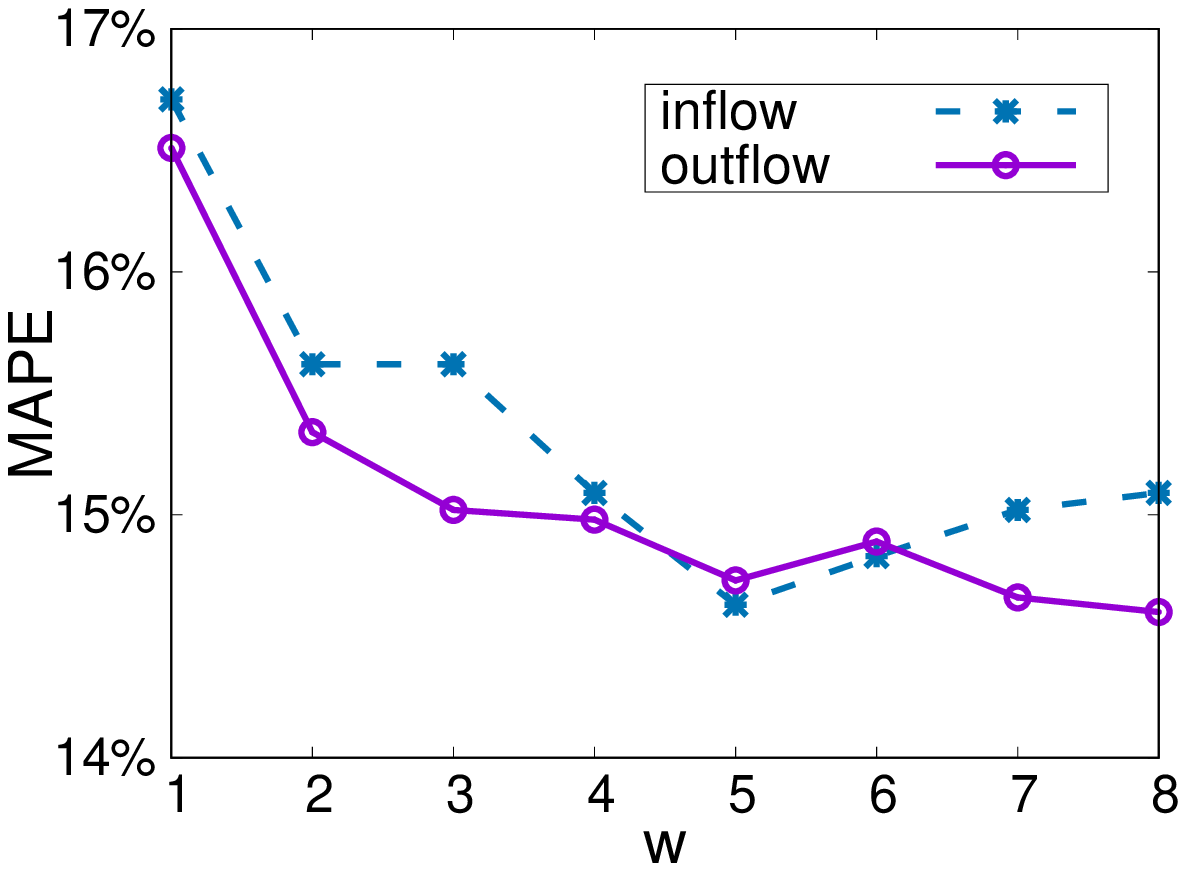}
        }
    }
    \caption{Impact of surrounding observations on the NYC-Taxi dataset.}
    \end{figure}

    \begin{figure}%[ht]
    \noindent\makebox[\linewidth][c]{
        \subfigure[\small RMSE.]{
            \label{fig:bike_spatial_len_rmse}
            \includegraphics[width=0.24\textwidth]{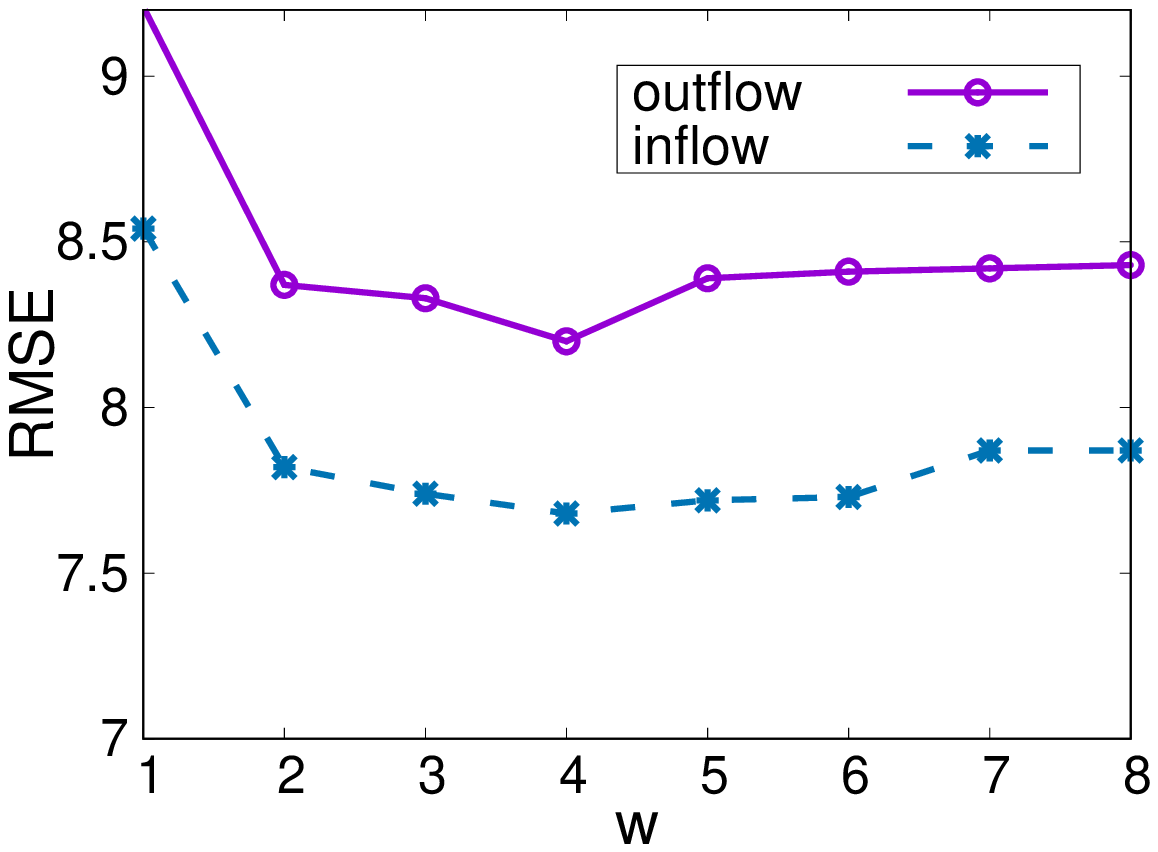}
        }
        %\vspace{.6in}
        \subfigure[\small MAPE.]{
            \label{fig:bike_spatial_len_mape}
            \includegraphics[width=0.24\textwidth]{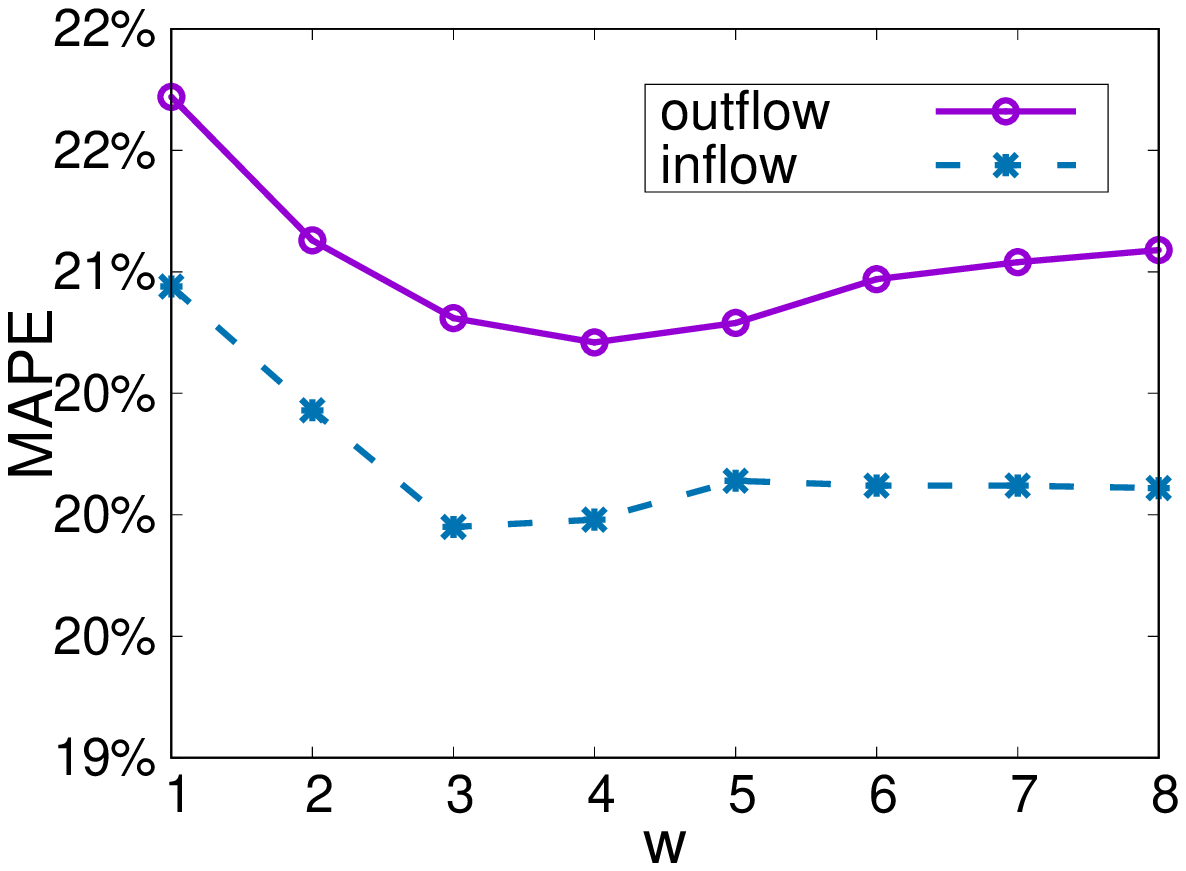}
        }
    }
    \caption{Impact of surrounding observations on the NYC-Bike dataset.}
    \end{figure}

Recall that the dependency between two regions may have temporal lagging due to the travel time between them. To capture such lagging dependency, \modelname{} uses the surrounding observations to learn the region correlations for each time slot, which contains the inflow and outflow in the previous $w$ time slots.
We thus evaluate the impact of $w$ on the performance of our model. 

Figures \ref{fig:taxi_spatial_len_rmse} and \ref{fig:taxi_spatial_len_mape} show the RMSE and MAPE versus different lengths on the NYC-Taxi dataset. 
%The length of $1$ means that only the flow in the time slot is considered, and hence no surrounding observation is used.
A larger length $w$ indicates that more information is encoded from the surrounding observations. When $w = 1$, only the observation at a time slot is used for dependency learning, and the model fails to capture the lagging characteristic of dependency. 
As the length increases, the RMSE and MAPE of both inflow and outflow forecasting decrease~($w \leq 5$). 
The performance improvement demonstrates the importance of using the surrounding observations to learn the dependencies between regions.
%and to encode historical flow information for lagging spatial correlation modeling.  
RMSE and MAPE increase slightly but remain stable when the length is large~($w \ge 5$). The potential reason is that, when $w$ is larger than the travel time between regions, increasing $w$ would not introduce more information for dependency learning. On the other hand, a larger $w$ may introduce some noise and more parameters for the model, leading to difficulties in model training~\cite{lin2020preserving}. The RMSE and MAPE versus different lengths of surrounding observations on the NYC-Bike dataset are shown in Figures \ref{fig:bike_spatial_len_rmse} and \ref{fig:bike_spatial_len_mape}, which are consistent with the results for the NYC-Taxi dataset. When the length is small~($w \leq 4$), RMSE and MAPE decrease as the length becomes larger, but they slightly increase when the length is large~($w \ge 4$). 
%The change point of $w$ for the taxi dataset is larger than the bike dataset. The reason could be that the duration of a taxi trip is usually larger than a bike trip.

\subsection{Layer Number}
\label{sec:layer_num}
\begin{figure}%[ht]
    \noindent\makebox[\linewidth][c]{
        \subfigure[\small RMSE.]{
            \label{fig:taxi_layer_num_rmse}
            \includegraphics[width=0.24\textwidth]{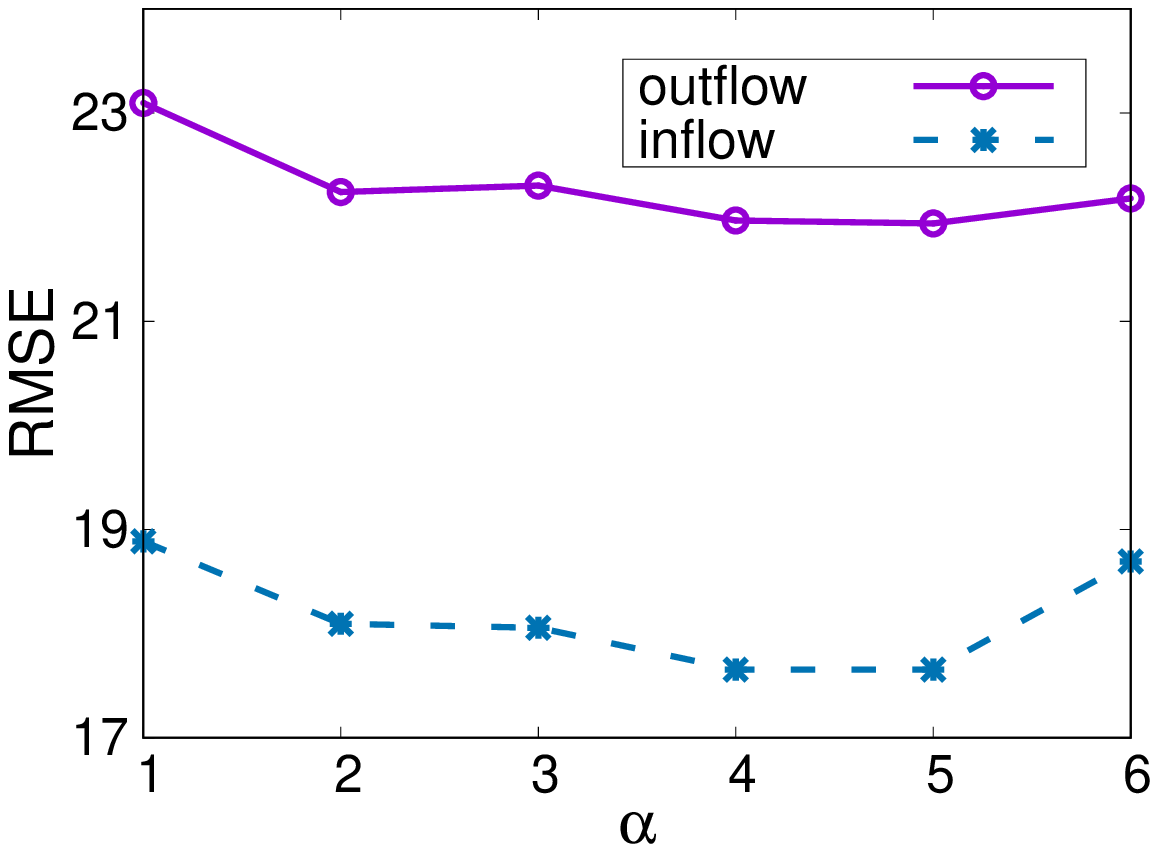}
        }
        %\vspace{.6in}
        \subfigure[\small MAPE.]{
            \label{fig:taxi_layer_num_mape}
            \includegraphics[width=0.24\textwidth]{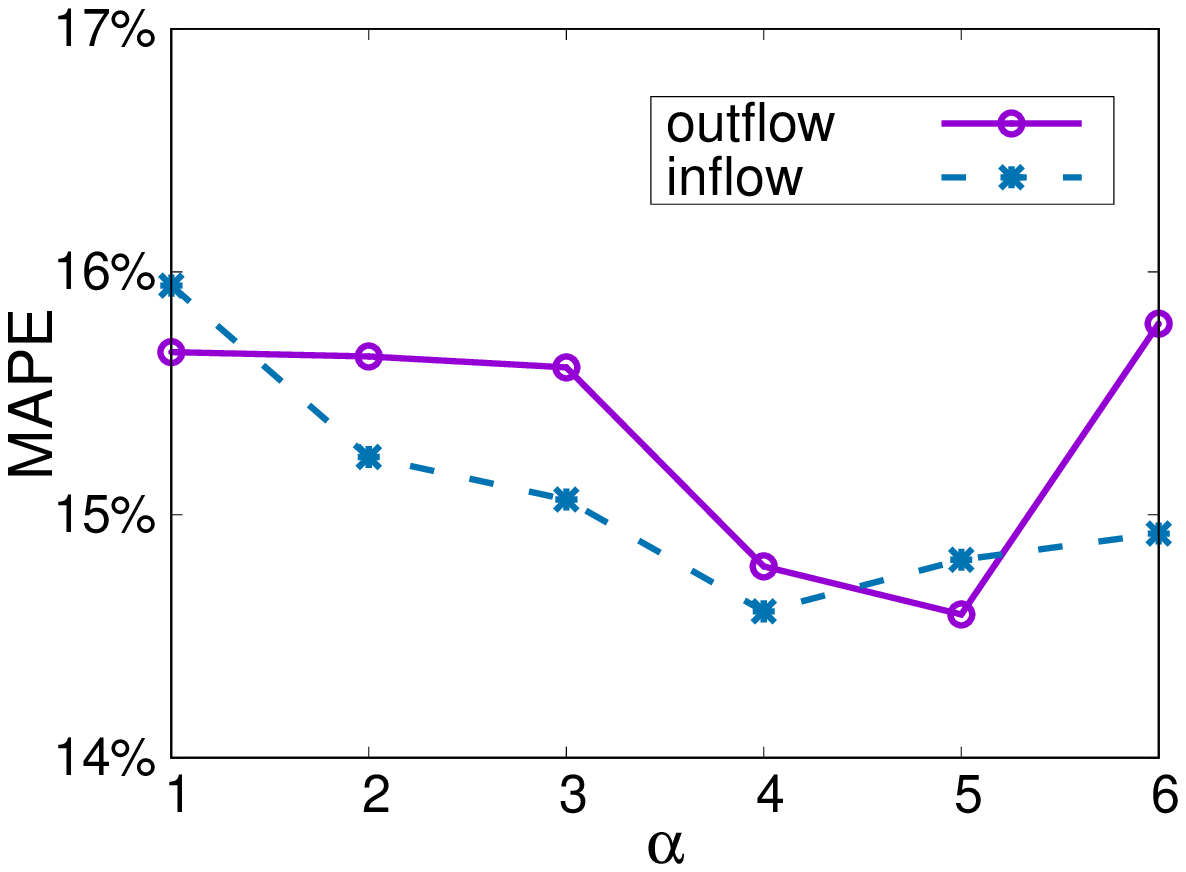}
        }
    }
    
    \caption{Impact of layer number on the NYC-Taxi dataset.}
    \label{fig:layer_number_taxi}
    \end{figure}
    
    \begin{figure}%[ht]
    \noindent\makebox[\linewidth][c]{
        \subfigure[\small RMSE.]{
            \label{fig:bike_layer_num_rmse}
            \includegraphics[width=0.24\textwidth]{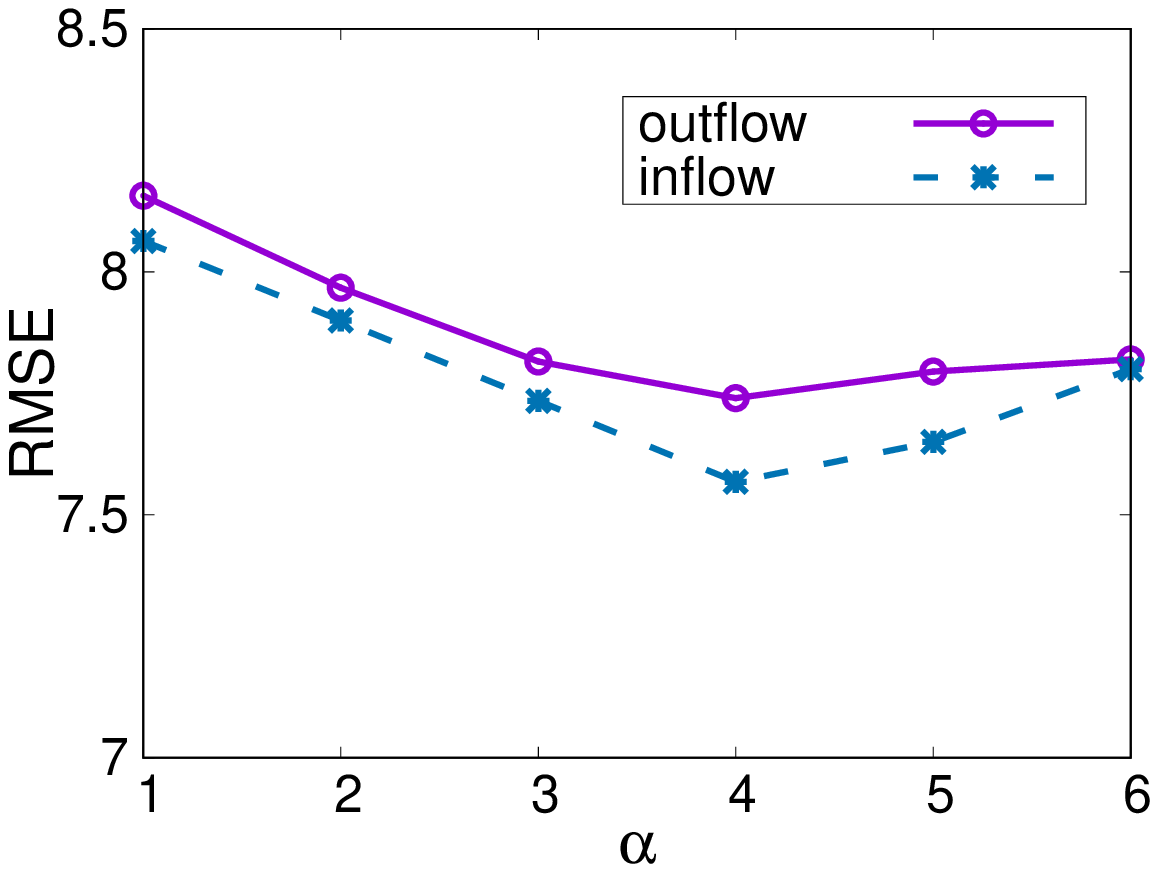}
        }
        %\vspace{.6in}
        \subfigure[\small MAPE.]{
            \label{fig:bike_layer_num_mape}
            \includegraphics[width=0.24\textwidth]{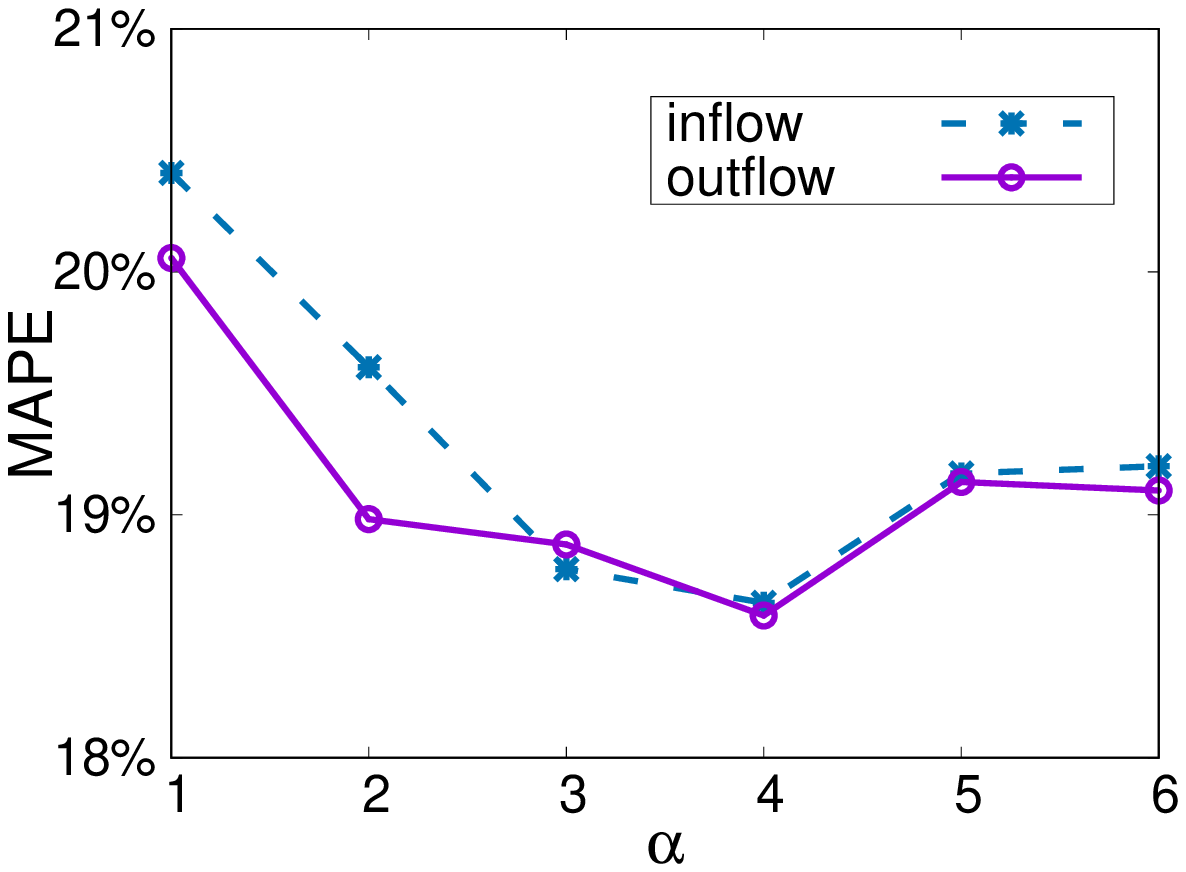}
        }
    }
    
    \caption{Impact of layer number on the NYC-Bike dataset.}
    \label{fig:layer_number_bike}
    \end{figure}

In \modelname{}, DLI is stacked $\alpha$ times to ensure the information propagtion between regions and improve the model robustness. We evaluate the effect of $\alpha$ on the prediction performance using RMSE and MAPE. The results for the taxi dataset are presented in Figure \ref{fig:layer_number_taxi}. When $\alpha = 1$, only the dependencies on one's neighbouring regions in the graph are considering, resulting in the worst prediction accuracy. When $\alpha \geq 2$, the dependencies on all other regions could be captured. We find that with the layer number $\alpha$ increases, the RMSE and the MAPE declines for both inflow and outflow, indicating the performance improvement. 
However, we find that
 when the layer number is too large~($\alpha > 5$ in our experiments), the performance on RMSE and MAPE degrades, because too many layers may result in difficulties of model training. Similar results are observed on the bike dataset~(Figure \ref{fig:layer_number_bike}).

\subsection{Head Number}
\label{sec:head_num}

We use the multi-head attention mechanism in \modelname{} for dependency learning, so that different heads could capture different patterns from the historical data. In our experiments, we evaluate the impact of the head number $M$ on the prediction performance. The results of RMSE and MAPE versus the head number $M$ on the taxi and bike datasets are presented in Figures \ref{fig:head_number_taxi} and \ref{fig:head_number_bike}, respectively. As shown in Figures \ref{fig:taxi_head_num_rmse} and \ref{fig:bike_head_num_rmse}, with the head number increases, the RMSE on the two datasets declines, demonstrating the multi-head mechanism could benefit the dependency learning and improve the prediction accuracy. We also observe that when the head number become larger~($M > 4$ on the taxi dataset while $M > 6$ on the bike dataset), the improvements are not significant. The reason could be that some heads may focus on the same pattern when there are many heads. In terms of MAPE, similar findings could be observed in Figures \ref{fig:taxi_head_num_mape} and \ref{fig:bike_head_num_mape}. Moreover, the MAPE increases slightly when the head number becomes large~($M > 6$). It is because increasing the head number leads to more learnable parameters, which would result in difficulties for model training.

\begin{figure}%[ht]
    \noindent\makebox[\linewidth][c]{
        \subfigure[\small RMSE.]{
            \label{fig:taxi_head_num_rmse}
            \includegraphics[width=0.24\textwidth]{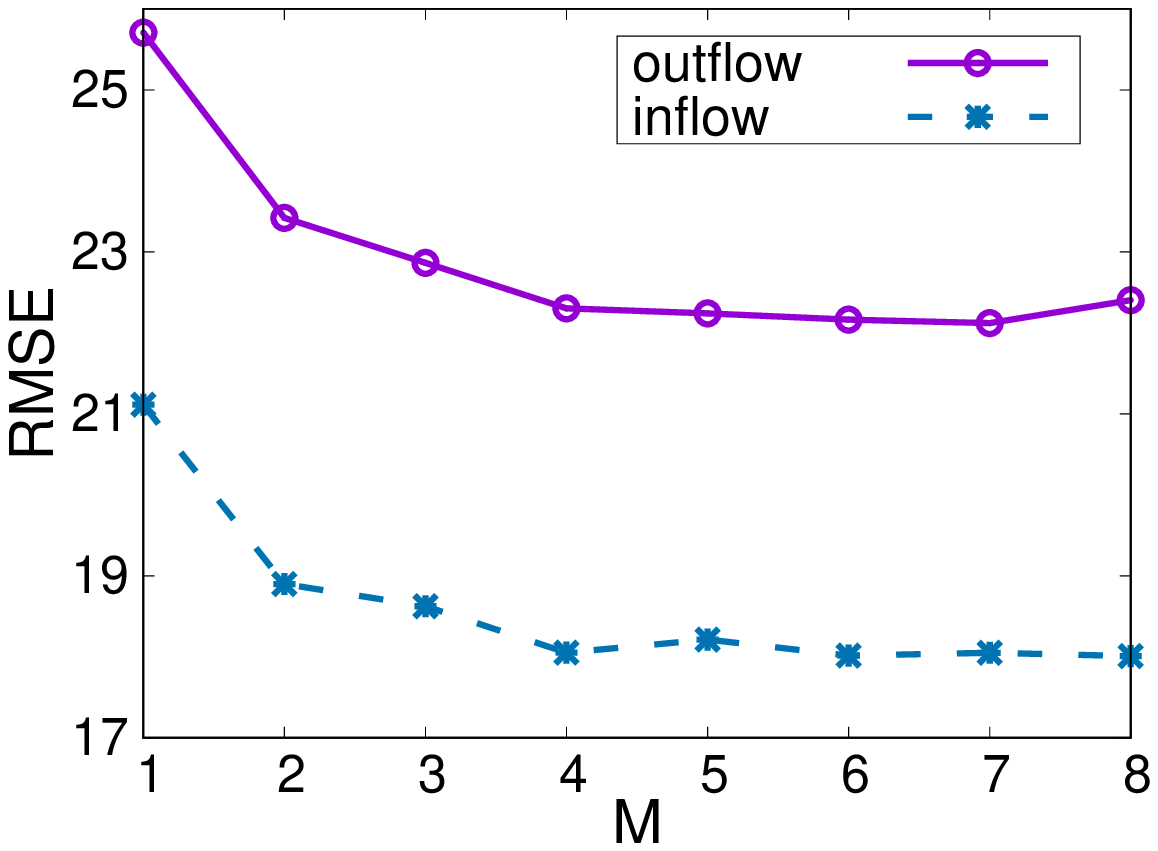}
        }
        %\vspace{.6in}
        \subfigure[\small MAPE.]{
            \label{fig:taxi_head_num_mape}
            \includegraphics[width=0.24\textwidth]{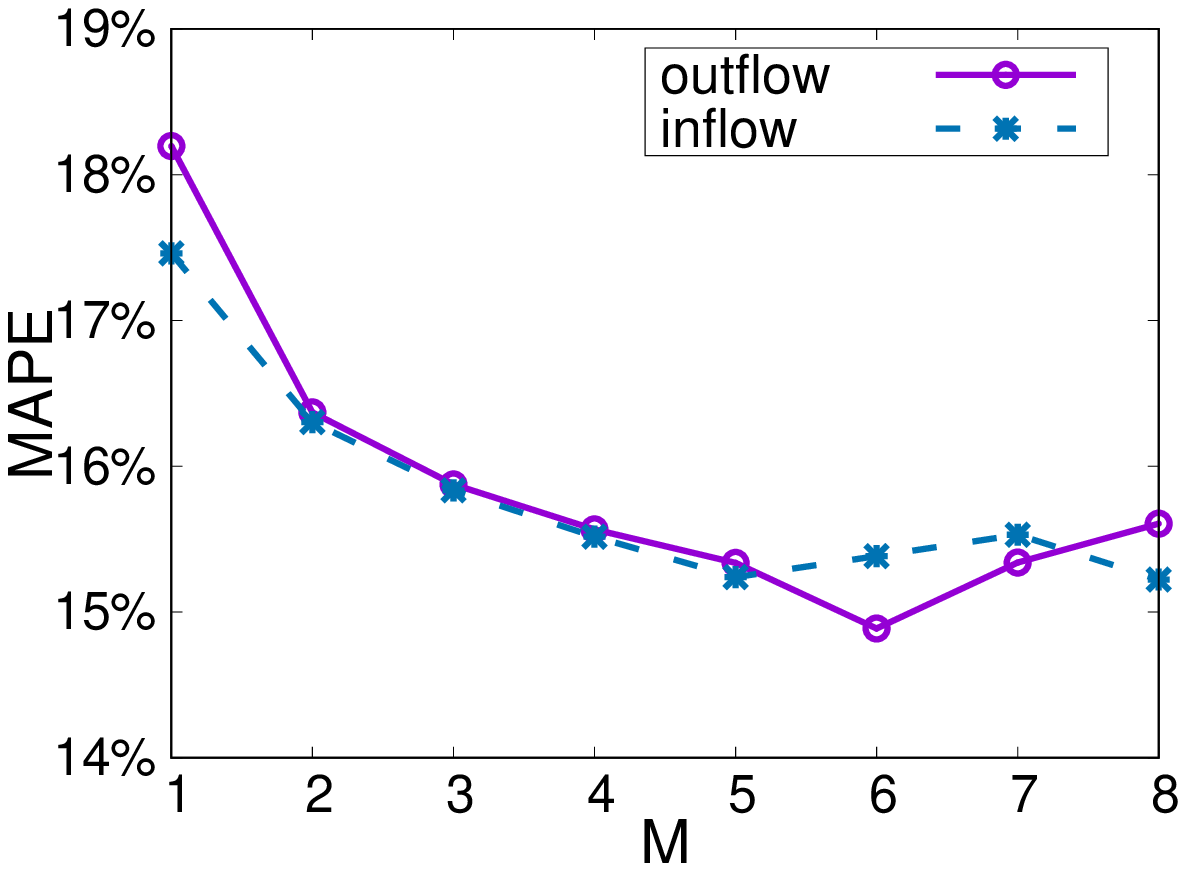}
        }
    }
    
    \caption{Impact of head number on the NYC-Taxi dataset.}
    \label{fig:head_number_taxi}
    \end{figure}
    
    \begin{figure}%[ht]
    \noindent\makebox[\linewidth][c]{
        \subfigure[\small RMSE.]{
            \label{fig:bike_head_num_rmse}
            \includegraphics[width=0.24\textwidth]{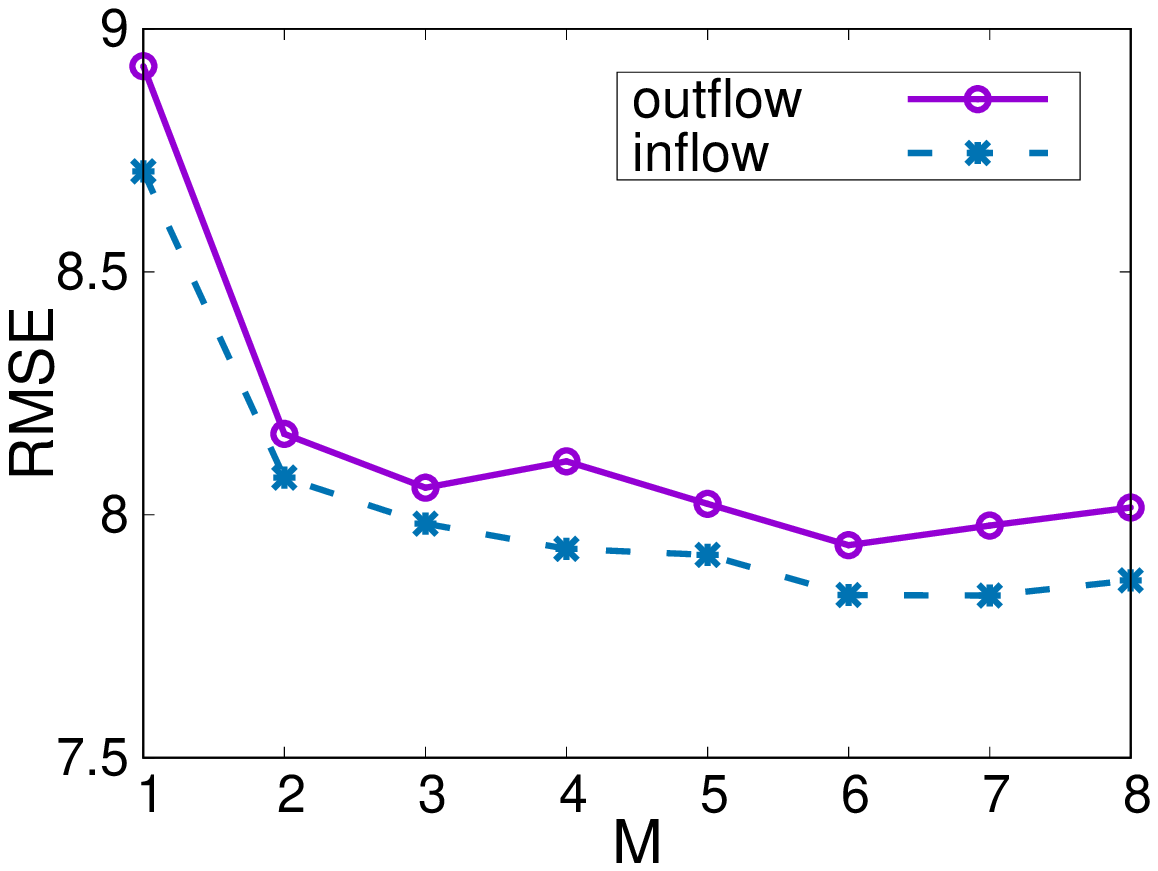}
        }
        %\vspace{.6in}
        \subfigure[\small MAPE.]{
            \label{fig:bike_head_num_mape}
            \includegraphics[width=0.24\textwidth]{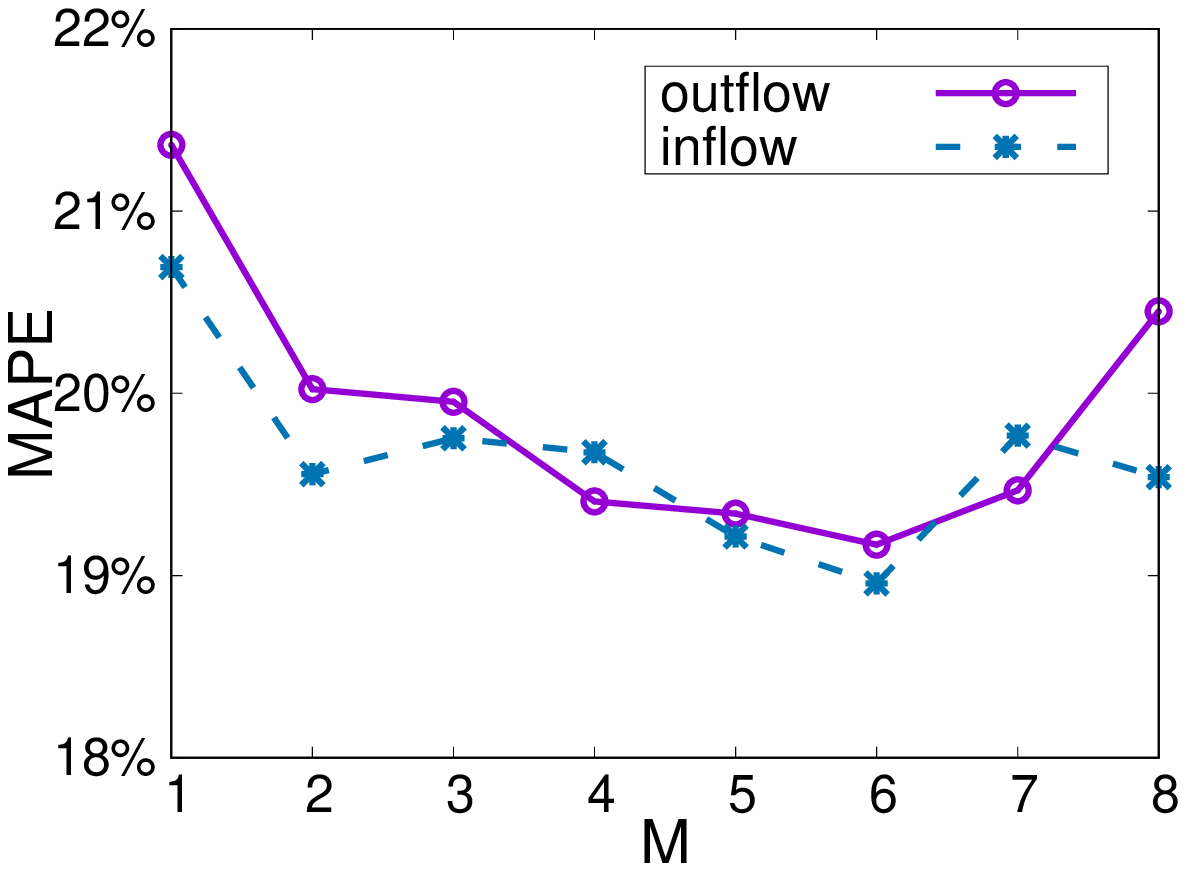}
        }
    }
    
    \caption{Impact of head number on the NYC-Bike dataset.}
    \label{fig:head_number_bike}
    \end{figure}

\section{Conclusion}
\label{sec:conclusion}
We 
propose \modelname{}, a novel, small (in parameters), computationally efficient and highly accurate model  for traffic forecasting.
%\modelname{} considers the joint, dynamic, and possibly periodic spatial-temporal dependency between regions.
%
\modelname{} employs a spatial-temporal Transformer with information fusion and region sampling to jointly consider the dynamic spatial and temporal dependencies between regions at any individual time slots, and also the possibly periodic spatial-temporal dependency from multiple time slots. In particular, \modelname{} boosts the efficiency and addresses the long-tail issue of the canonical Transformer using a novel region sampling strategy, which reduces the complexity from $O(n^2)$ to $O(n\sqrt{n})$, where $n$ is the number of regions.
We have conducted extensive experiments
to evaluate \modelname{}, using a taxi and a bike sharing datasets. Our experimental results show that \modelname{} significantly outperforms the state-of-the-art approaches in terms of training efficiency~(with a reduction of $46\% \sim 95\%$ on training time and $23\% \sim 98\%$ on network parameters), and hence is efficient in tuning, training and memory.
Despite its small size and fast training,
it  achieves higher accuracy in its online predictions than other state-of-the-art works~(with improvement of up to $9.5\%$ on RMSE, and $12.4\%$ on MAPE). 
%Our case studies also demonstrate the interpretability and effectiveness of \modelname{} to capture the region correlations, temporal dependencies, and irregularity effect.

%Furthermore, our case studies demonstrate the interpretability and effectiveness of \modelname{} to capture the region correlation, temporal dependency, and irregularity effect.
%Furthermore, \modelname{} has fewer trainable parameters. Compared with the-state-of-the-art approaches, our model achieves comparable performance with less training data~($20$ days of data).

%\input{7_appendix}

%\section*{Acknowledgment}
%This work was supported, in part, by Hong Kong General Research Fund under grant number 16200120.

\bibliographystyle{IEEEtran}
\bibliography{main}
\end{document}